\documentclass[letterpaper, 10 pt, conference]{ieeeconf}  % Comment this line out if you need a4paper
\IEEEoverridecommandlockouts                              % This command is only needed if 
\usepackage{00_ps}

\preprinttrue% Either \preprinttrue or \preprintfalse
\publishedtrue% Either \publishedtrue or \publishedtrue

\glsdisablehyper % disable hyperlinks for glossary entries
\newcommand{\anAgent}{\ensuremath{i}}
\newcommand{\anotherAgent}{\ensuremath{j}}

\newcommand{\ofAgent}[1]{\ensuremath{^{(#1)}}}

\newglossaryentry{matrix:Adjacency}{
	name=\ensuremath{\bm{D}},
	description={Adjacency matrix},
	sort={D},
    type=symbol
}
\newcommand{\matAdjacency}{\gls{matrix:Adjacency}}
\newcommand{\matAdjacencyElement}[1]{\glslink{matrix:Adjacency}{\matAdjacency_{#1}}}

\newglossaryentry{set:realNumbers}{
	name=\ensuremath{\mathbb{R}},
	description={Set of real numbers},
	sort={real numbers},
    type=symbol
}
\newcommand{\setRealNumbers}{\gls{set:realNumbers}}

\newglossaryentry{set:naturalNumbers}{
	name=\ensuremath{\mathbb{N}},
	description={Set of natural numbers},
	sort={natural numbers},
    type=symbol
}

\newglossaryentry{set:systemStates}{
	name=\ensuremath{\mathcal{S}},
	description={Set of system states},
	sort={set of system states},
    type=symbol
}

\newglossaryentry{set:bigO}{
	name=\ensuremath{O},
	description={Big O},
	sort={O},
    type=symbol
}

\newglossaryentry{scalar:Weight}{
	name=\ensuremath{w},
	description={Weight},
	sort={weight},
    type=symbol
}

\newglossaryentry{scalar:NumberOfAgents}{
    name=\ensuremath{N},
    description={Number of agents},
    sort={Number of agents},
    type=symbol
}
\newcommand{\numAgents}{\gls{scalar:NumberOfAgents}}

\newglossaryentry{graph:path}{
    name=\ensuremath{\pi},
    description={Path},
    sort={Path},
    type=symbol
}
\newcommand{\graphPath}{\gls{graph:path}}

\newglossaryentry{scalar:NumberOfVerticesInPath}{
    name=\ensuremath{N_{\graphPath}},
    description={Number of vertices in path $\graphPath$, or length},
    sort={Number of vertices in path},
    type=symbol
}
\newcommand{\numVerticesPath}{\gls{scalar:NumberOfVerticesInPath}}

\newglossaryentry{trajectory:Reference}{
    name=\ensuremath{\bm{r}},
    description={Reference trajectory},
    sort={Reference Trajectory},
    type=symbol
}

\newglossaryentry{sym:horizonControl}{
	name=\ensuremath{N_u},
	description={Control horizon in model predictive control},
	sort={Nu},
    type=symbol
}

\newglossaryentry{sym:horizonPrediction}{
	name=\ensuremath{N_p},
	description={Prediction horizon in model predictive control},
	sort={Np},
    type=symbol
}

\newglossaryentry{sym:vehicleOrientation}{
	name=\ensuremath{\psi},
	description={Vehicle orientation},
	sort={psi},
    type=symbol
}

\newglossaryentry{sym:sysModelContinuous}{
    name=\ensuremath{f},
    description={Continuous-time system model},
    sort={f continuous-time},
    type=symbol
}
\newcommand{\sysModelContinuous}{\gls{sym:sysModelContinuous}}

\newglossaryentry{sym:sysModelDiscrete}{
    name=\ensuremath{f_{d}},
    description={Discrete-time system model},
    sort={f discrete-time},
    type=symbol
}

\newglossaryentry{sym:sysControlInputs}{
	name=\ensuremath{\bm{u}},
	description={System control inputs},
	sort=u,
    type=symbol
}
\NewDocumentCommand{\sysControlInputs}{ o }{\glslink{sym:sysControlInputs}{%
    \IfNoValueTF{#1}%
        {\ensuremath{\bm{u}}}%
        {\ensuremath{\bm{u}^{(#1)}}}%
}}

\newglossaryentry{sym:outputs}{
	name=\ensuremath{\bm{y}},
	description={System outputs},
	sort={y},
    type=symbol
}

\newglossaryentry{sym:sysSpeed}{
	name=\ensuremath{\mathrm{v}},
	description={Vehicle speed},
	sort={v},
    type=symbol
}
\newcommand{\sysSpeed}{\gls{sym:sysSpeed}}

\newglossaryentry{sym:inSpeed}{
	name=\ensuremath{u_{\sysSpeed}},
	description={Vehicle input speed},
	sort={uv},
    type=symbol
}

\newglossaryentry{sym:steering-angle}{
	name=\ensuremath{\delta},
	description={Vehicle steering angle},
	sort={delta},
    type=symbol
}

\newglossaryentry{sym:inSteering}{
	name=\ensuremath{u_{\delta}},
	description={Vehicle input steering angle},
	sort={ud},
    type=symbol
}

\newglossaryentry{sym:nColors}{
	name=\ensuremath{N_c},
	description={Number of colors},
	sort={Number of colors},
    type=symbol
}

\newglossaryentry{sym:nStates}{
	name=\ensuremath{n},
	description={Number of states of a dynamical system},
	sort={Number of states},
    type=symbol
}
\newcommand{\numStates}{\gls{sym:nStates}}

\newglossaryentry{sym:nInputs}{
    name=\ensuremath{m},
    description={Number of inputs of a dynamical system},
    sort={m number of inputs},
    type=symbol
}
\newcommand{\numInputs}{\gls{sym:nInputs}}

\newglossaryentry{sym:nLevels}{
	name=\ensuremath{N_{\text{CL}}},
	description={Number of computation levels},
	sort={Number of computation levels},
    type=symbol
}

\newglossaryentry{sym:nLevelsAllowed}{
	name=\ensuremath{N_{\text{CL},a}},
	description={Allowed number of computation levels},
	sort={Number of computation levels allowed},
    type=symbol
}

\newglossaryentry{sym:numGroups}{
	name=\ensuremath{N_{g}},
	description={Number of parallelly computing groups of agents},
	sort={Number of groups},
    type=symbol
}

\newglossaryentry{sym:fnPrio}{
    name=\ensuremath{p},
    description={Priority assignment function},
    sort={Priority assignment function},
    type=symbol
}

\newglossaryentry{sym:tSample}{
	name=\ensuremath{T_s},
	description={Sample Time},
	sort={T sample},
    type=symbol
}

\newglossaryentry{sym:tSolve}{
	name=\ensuremath{T_\text{sol.}},
	description={Computation time \tSolveB{\anAgent} that agent $\anAgent$ needs to solve its \ac{ocp}},
	sort={T solve},
    type=symbol
}

\newcommand{\tSolveB}[1]{\glslink{sym:tSolve}{\ensuremath{\ensuremath{T_\text{sol.}}^{(#1)}}}}

\newglossaryentry{sym:tSolveUpper}{
	name=\ensuremath{T_\text{sol.,max}},
	description={Upper computation time $T_\text{sol.,max}\ofAgent{\anAgent}$ that agent $\anAgent$ needs to solve it \ac{ocp}},
	sort={T solve upper},
    type=symbol
}

\newglossaryentry{sym:vertices}{
	name=\ensuremath{\mathcal{V}},
	description={Set of vertices},
	sort={Vertices},
    type=symbol
}
\newcommand{\setVertices}{\gls{sym:vertices}}
\newcommand{\setAgents}{\setVertices}

\newcommand{\helpSetPredecessors}[1]{\ensuremath{\setVertices^{(#1\leftarrow)}}}
\newglossaryentry{sym:predecessors}{
	name=\ensuremath{\helpSetPredecessors{i}},
	description={Set of predecessors of vertex $i$},
	sort={Vertices 1},
    type=symbol
}
\newcommand{\setPredecessors}[1]{\glslink{sym:predecessors}{\ensuremath{\helpSetPredecessors{#1}}}}

\newcommand{\helpSetPredecessorsPar}[1]{\ensuremath{\setVertices^{(#1\leftarrow)}_{\text{par.}}}}
\newglossaryentry{sym:predecessorsPar}{
	name=\ensuremath{\helpSetPredecessorsPar{i}},
	description={Set of predecessors of vertex $i$ that have parallel couplings with it},
	sort={Vertices 2},
    type=symbol
}

\newcommand{\helpSetPredecessorsSeq}[1]{\ensuremath{\setVertices^{(#1\leftarrow)}_{\text{seq.}}}}
\newglossaryentry{sym:predecessorsSeq}{
	name=\ensuremath{\helpSetPredecessorsSeq{i}},
	description={Set of predecessors of vertex $i$ that have sequential couplings with it},
	sort={Vertices 3},
    type=symbol
}

\newcommand{\helpSetSuccessors}[1]{\ensuremath{\setVertices^{(#1\rightarrow)}}}
\newglossaryentry{sym:successors}{
	name=\ensuremath{\helpSetSuccessors{i}},
	description={Set of successors of vertex $i$},
	sort={Vertices 4},
    type=symbol
}
\newcommand{\setSuccessors}[1]{\glslink{sym:successors}{\ensuremath{\helpSetSuccessors{#1}}}}

\newglossaryentry{sym:neighbors}{
	name=\ensuremath{\setVertices^{(i)}},
	description={Set of neighbors of vertex $i$},
	sort={Vertices 0},
    type=symbol
}
\newcommand{\setNeighbors}[1]{\glslink{sym:neighbors}{\ensuremath{\setVertices^{(#1)}}}}

\newglossaryentry{sym:degree}{
	name=\ensuremath{d^{(i)}},
	description={Degree of vertex $i$. Sum of in-degree and out-degree},
	sort=degree,
    type=symbol
}
\newcommand{\vertexDegree}[1]{\glslink{sym:degree}{\ensuremath{d^{(#1)}}}}

\newcommand{\helpVertexInDegree}[1]{\ensuremath{d^{(#1\leftarrow)}}}
\newglossaryentry{sym:inDegree}{
    name=\helpVertexInDegree{i},
    description={In-degree of vertex $i$},
    sort={degree in},
    type=symbol
}
\newcommand{\vertexInDegree}[1]{\glslink{sym:inDegree}{\helpVertexInDegree{#1}}}

\newcommand{\helpVertexOutDegree}[1]{\ensuremath{d^{(#1\rightarrow)}}}
\newglossaryentry{sym:outDegree}{
    name=\helpVertexOutDegree{i},
    description={Out-degree of vertex $i$},
    sort={degree out},
    type=symbol,
}
\newcommand{\vertexOutDegree}[1]{\glslink{sym:outDegree}{\helpVertexOutDegree{#1}}}

\newglossaryentry{sym:matLevels}{
	name=\ensuremath{\bm{L}},
	description={Matrix of computation levels},
	sort=L,
    type=symbol
}

\newglossaryentry{sym:tComp}{
	name=\ensuremath{T},
	description={Computation time},
	sort={T},
    type=symbol
}

\newglossaryentry{sym:tCompNcs}{
	name=\ensuremath{T_{\text{NCS}}},
	description={Computation time of \iac{ncs}},
	sort={T NCS},
    type=symbol
}
\newcommand{\tCompNcs}{\gls{sym:tCompNcs}}

\newglossaryentry{graph:Undirected}{
	name=\ensuremath{\mathcal{G}},
	description={Undirected Graph},
	sort={graph1},
    type=symbol
}
\newcommand{\graphUndirected}{\gls{graph:Undirected}}

\newglossaryentry{graph:Directed}{
    name=\ensuremath{\vec{\gls*{graph:Undirected}}},
	description={Directed Graph},
	sort={graph2},
    type=symbol
}
\newcommand{\graphDirected}{\gls{graph:Directed}}

\newglossaryentry{mat:edgeUtilities}{
    name=\ensuremath{M_\text{u}},
	description={Edge utility matrix},
	sort={matrix edge utilities},
    type=symbol
}

\newglossaryentry{sym:setColors}{
    name=\ensuremath{\mathcal{C}},
    description={Set of colors},
    sort=Colors,
    type=symbol
}

\newglossaryentry{sym:varControlInvariantSet}{
	name=\ensuremath{\mathcal{C}_{\text{inv}}},
	description={Control invariant set},
	sort={Control invariant set},
    type=symbol
}

\newglossaryentry{set:Weights}{
	name=\ensuremath{\mathcal{W}},
	description={Set of weights in a weighted graph},
    sort={Weights},
    type=symbol
}

\newglossaryentry{set:Edges}{
	name=\ensuremath{\mathcal{E}},
	description={Set of edges; used to indicate that only undirected edges exist},
    sort={Edges},
    type=symbol
}
\newcommand{\setEdges}{\gls{set:Edges}}

\newglossaryentry{sym:setEdgesDirected}{
	name=\ensuremath{\vec{\gls*{set:Edges}}},
	description={Set of directed edges},
	sort={Edges directed},
    type=symbol
}

\newglossaryentry{sym:varEdge}{
	name=\ensuremath{(i \rightarrow j)},
	description={Directed edge from vertex $i$ to vertex $j$},
	sort={edge},
    type=symbol
}
\newcommand{\edgeDirected}[2]{\glslink{sym:varEdge}{\ensuremath{(#1 \rightarrow #2)}}}

\newglossaryentry{sym:fnReorder}{
	name=\ensuremath{f_r},
	description={Reordering function for graph color values},
	sort=fr,
    type=symbol
}

\newglossaryentry{sym:fcnObjective}{
    name=\ensuremath{J},
    description={Objective function of an optimization problem},
    sort=J,
    type=symbol
}
\NewDocumentCommand{\fcnObjective}{ o }{\glslink{sym:fcnObjective}{%
    \IfNoValueTF{#1}%
        {\ensuremath{J}}%
        {\ensuremath{J^{(#1)}}}%
}}

\newglossaryentry{sym:fcnObjectiveState}{
    name=\ensuremath{\ell_{x}},
    description={Reference tracking objective function},
    sort={lx Reference tracking objective function},
    type=symbol
}
\NewDocumentCommand{\fcnObjectiveState}{ o }{\glslink{sym:fcnObjectiveState}{%
    \IfNoValueTF{#1}%
        {\ensuremath{\ell_{x}}}%
        {\ensuremath{\ell_{x}^{(#1)}}}%
}}

\newglossaryentry{sym:fcnObjectiveStateTerminal}{
    name=\ensuremath{\ell_{f}},
    description={Reference tracking objective terminal function},
    sort={lf Reference tracking objective terminal function},
    type=symbol
}

\newglossaryentry{sym:fcnObjectiveInput}{
    name=\ensuremath{\ell_{u}},
    description={Input change objective function},
    sort={lu Input change objective function},
    type=symbol
}
\NewDocumentCommand{\fcnObjectiveInput}{ o }{\glslink{sym:fcnObjectiveInput}{%
    \IfNoValueTF{#1}%
        {\ensuremath{\ell_{u}}}%
        {\ensuremath{\ell_{u}^{(#1)}}}%
}}

\newglossaryentry{sym:fcnObjectiveCoupling}{
    name=\ensuremath{\ell_\text{c}},
    description={Coupling objective function},
    sort={lc Coupling objective function},
    type=symbol
}
\NewDocumentCommand{\fcnObjectiveCoupling}{ oo }{\glslink{sym:fcnObjectiveCoupling}{%
    \IfNoValueTF{#1}%
        {\ensuremath{\ell_\text{c}}}%
        {\ensuremath{\ell_\text{c}^{(#1,#2)}}}%
}}

\newglossaryentry{sym:fcnConstraintCoupling}{
    name=\ensuremath{c_\text{c}},
    description={Coupling constraint function},
    sort={cc Coupling constraint function},
    type=symbol
}
\NewDocumentCommand{\fcnConstraintCoupling}{ oo }{\glslink{sym:fcnConstraintCoupling}{%
    \IfNoValueTF{#1}%
        {\ensuremath{c_\text{c}}}%
        {\ensuremath{c_\text{c}^{(#1,#2)}}}%
}}

\newglossaryentry{sym:prediction}{
	name=\ensuremath{\tilde{\bm{x}}^{(j \leftarrow i)}_{\cdot \vert k}},
	description={Prediction in agent $i$ for agent $j$ at time $k$},
	sort=x,
    type=symbol,
}
\newcommand{\agentPrediction}{\glslink{sym:prediction}{\ensuremath{ \tilde{\bm{x}} }}}

\newcommand{\agentPredictionForAgentAFromAgentBAtTimeC}[3]{\glslink{sym:prediction}{\ensuremath{ \agentPrediction^{(#1 \leftarrow #2)}_{\cdot \vert #3} }}}

\newglossaryentry{sym:state}{
	name=\ensuremath{\bm{x}},
	description={System state},
	sort=x,
    type=symbol
}
\newcommand{\sysState}{\gls{sym:state}}
\newglossaryentry{sym:ref}{
	name=\ensuremath{\bm{r}},
	description={System reference},
	sort=r,
    type=symbol
}

\newglossaryentry{sym:stateAgent}{
	name=\ensuremath{\sysState^{(i)}_{(k)}},
	description={System state of agent $i$ at time $k$},
	sort=x,
    type=symbol,
}

\newglossaryentry{sym:setReachable}{
	name=\ensuremath{\mathcal{R}^{(i)}},
	description={reachable set of agent $i$},
	sort={Reachable set},
    type=symbol
}
\newcommand{\setReachable}{\glslink{sym:setReachable}{\ensuremath{\mathcal{R}}}}

\newglossaryentry{set:occupiedArea}{
	name=\ensuremath{\mathcal{O}^{(i)}},
	description={Set of the occupied area of the predicted trajectory of agent $\anAgent$},
	sort={occupied area},
    type=symbol
}

\newglossaryentry{set:feasibleStates}{
	name=\ensuremath{\mathcal{X}},
	description={set of feasible states},
	sort={x},
    type=symbol
}
\newcommand{\setFeasibleStates}{\gls{set:feasibleStates}}

\newglossaryentry{set:feasibleInputs}{
	name=\ensuremath{\mathcal{U}},
	description={set of feasible inputs},
	sort={u},
    type=symbol
}
\newcommand{\setFeasibleInputs}{\gls{set:feasibleInputs}}

\newglossaryentry{sym:numStatesConfSpace}{
    name=\ensuremath{n_p},
    description={Number of states that are in the conflictual space},
    sort={n number of states that are in the conflictual space},
    type=symbol
}

\newglossaryentry{sym:fcnProj}{
    name=\text{proj},
    description={A function that projects a reachable set of system states in the conflictual space},
    sort={Project function},
    type=symbol
}

\newglossaryentry{rl:setOfAgents}{
    name=\ensuremath{\mathcal{N}},
    description={A set of agents},
    sort={Set of agents},
    type=symbol
}

\newglossaryentry{rl:actionSpace}{
    name=\ensuremath{\mathcal{A}},
    description={Action space},
    sort={Action space},
    type=symbol
}

\newglossaryentry{rl:stateSpace}{
    name=\ensuremath{\mathcal{S}},
    description={State space},
    sort={State space},
    type=symbol
}

\newglossaryentry{rl:observationSpace}{
    name=\ensuremath{\mathcal{O}},
    description={Observation space},
    sort={Observation space},
    type=symbol
}

\newglossaryentry{rl:policySpace}{
    name=\ensuremath{\Omega},
    description={Policy space},
    sort={Policy space},
    type=symbol
}

\newglossaryentry{rl:observationFcn}{
    name=\ensuremath{\Omega},
    description={Observation function},
    sort={Observation function},
    type=symbol
}

\newglossaryentry{rl:transitionProbFcn}{
    name=\ensuremath{\mathcal{P}},
    description={Transition probability function},
    sort={Transition probability function},
    type=symbol
}

\newglossaryentry{rl:rewardFcn}{
    name=\ensuremath{R},
    description={Reward function},
    sort={Reward function},
    type=symbol
}

\newglossaryentry{rl:discountFactor}{
    name=\ensuremath{\gamma},
    description={Discount factor},
    sort={Discount factor},
    type=symbol
}

\newglossaryentry{rl:state}{
    name=\ensuremath{s},
    description={State},
    sort={State},
    type=symbol
}

\newglossaryentry{rl:nextState}{
    name=\ensuremath{s'},
    description={Next state},
    sort={Next state},
    type=symbol
}

\newglossaryentry{rl:action}{
    name=\ensuremath{a},
    description={Action},
    sort={Action},
    type=symbol
}

\newglossaryentry{rl:jointActions}{
    name=\ensuremath{a},
    description={Joint actions},
    sort={Joint actions},
    type=symbol
}

\newglossaryentry{rl:observation}{
    name=\ensuremath{o},
    description={Observation},
    sort={Observation},
    type=symbol
}

\newglossaryentry{rl:jointObservations}{
    name=\ensuremath{\bm{o}},
    description={Joint observations},
    sort={Joint observations},
    type=symbol
}

\newglossaryentry{rl:reward}{
    name=\ensuremath{r},
    description={Reward},
    sort={Reward},
    type=symbol
}

\newglossaryentry{rl:jointRewards}{
    name=\ensuremath{\bm{r}},
    description={Joint rewards},
    sort={Joint rewards},
    type=symbol
}

\newglossaryentry{rl:valueFunction}{
    name=\ensuremath{V},
    description={A function that evaluates how good a policy is},
    sort={Value function},
    type=symbol
}

\newglossaryentry{rl:policy}{
    name=\ensuremath{\pi},
    description={Policy},
    sort={Policy},
    type=symbol
}

\newglossaryentry{rl:policyOptimal}{
    name=\ensuremath{\pi^{*}{}},
    description={Optimal policy},
    sort={Optimal policy},
    type=symbol
}

\newcommand{\baseSubscript}{\textbf{D}} % subscript for the base problem
\newcommand{\priSubscript}{\textbf{P}} % subscript for the priority-assignment problem

\newglossaryentry{rl:setOfPolicies}{
    name=\ensuremath{\Pi},
    description={Set of policies},
    sort={Set of policies},
    type=symbol
}
\newcommand{\setOfPolicies}{\gls{rl:setOfPolicies}}

\newglossaryentry{rl:setOfPoliciesBase}{
    name=\ensuremath{\setOfPolicies_{\baseSubscript}},
    description={Set of policies of the decision-making policy},
    sort={Set of policies of the decision-making policy},
    type=symbol
}

\newglossaryentry{rl:setOfPoliciesPri}{
    name=\ensuremath{\setOfPolicies_{\priSubscript}},
    description={Set of policies of the priority-assignment policy},
    sort={Set of policies of the priority-assignment policy},
    type=symbol
}

\newglossaryentry{rl:setOfActions}{
    name=\ensuremath{A},
    description={Set of actions},
    sort={Set of actions},
    type=symbol
}
\newcommand{\setOfActions}{\gls{rl:setOfActions}}

\newglossaryentry{rl:setOfActionsBase}{
    name=\ensuremath{\setOfActions_{\baseSubscript}},
    description={Set of actions of the decision-making policy},
    sort={Set of actions of the decision-making policy},
    type=symbol
}

\newglossaryentry{rl:setOfActionsPri}{
    name=\ensuremath{\setOfActions_{\priSubscript}},
    description={Set of actions of the priority-assignment policy},
    sort={Set of actions of the priority-assignment policy},
    type=symbol
}

\newglossaryentry{rl:setOfObservations}{
    name=\ensuremath{O},
    description={Set of observations},
    sort={Set of observations},
    type=symbol
}
\newcommand{\setOfObservations}{\gls{rl:setOfObservations}}

\newglossaryentry{rl:setOfObservationsBase}{
    name=\ensuremath{\setOfObservations_{\baseSubscript}},
    description={Set of observations of the decision-making policy},
    sort={Set of observations of the decision-making policy},
    type=symbol
}

\newglossaryentry{rl:setOfObservationsPri}{
    name=\ensuremath{\setOfObservations_{\priSubscript}},
    description={Set of observations of the priority-assignment policy},
    sort={Set of observations of the priority-assignment policy},
    type=symbol
}

\newglossaryentry{sym:distance}{
    name=\ensuremath{d},
    description={Distance},
    sort={Distance},
    type=symbol
}

\newglossaryentry{sym:numOfPointsRef}{
    name=\ensuremath{n_\text{p,RP}},
    description={Number of points on the reference path},
    sort={Number of points on the reference path},
    type=symbol
}

\newglossaryentry{sym:numOfObservedSurroundingAgents}{
    name=\ensuremath{n_\text{sur.}},
    description={Number of observed surrounding agents},
    sort={Number of observed surrounding agents},
    type=symbol
}

\newglossaryentry{sym:numOfNotMaskedSurroundingAgents}{
    name=\ensuremath{n_\text{sur.,NM}},
    description={Number of not masked surrounding agents},
    sort={Number of not masked agents},
    type=symbol
}

\newglossaryentry{sym:numOfMaskedSurroundingAgents}{
    name=\ensuremath{n_\text{sur.,M}},
    description={Number of masked surrounding agents},
    sort={Number of masked agents},
    type=symbol
}

\newglossaryentry{rl:numOfSamples}{
    name=\ensuremath{n_\text{samples}},
    description={Number of training samples},
    sort={Number of training samples},
    type=symbol
}

\newglossaryentry{rl:setOfModels}{
    name=\ensuremath{\mathcal{M}},
    description={A set of models},
    sort={A set of models},
    type=symbol
}

\newglossaryentry{rl:maxModelIndex}{
    name=\ensuremath{5},
    description={Maximum Model Index},
    sort={Maximum Model Index},
    type=symbol
}

\newglossaryentry{rl:numberOfModels}{
    name=\text{six},
    description={Number of models},
    sort={Number of models},
    type=symbol
}

\newglossaryentry{rl:model}{
    name=\ensuremath{M},
    description={RL Model},
    sort={RL Model},
    type=symbol
}

\newglossaryentry{rl:compositeScore}{
    name=\ensuremath{CS},
    description={Composite Score},
    sort={Composite Score},
    type=symbol
}

\newglossaryentry{rl:collisionRate}{
    name=\ensuremath{CR},
    description={Collision Rate},
    sort={Collision Rate},
    type=symbol
}

\newglossaryentry{rl:collisionRateTotal}{
    name=\ensuremath{CR_{\text{total}}},
    description={Total Collision Rate},
    sort={Total Collision Rate},
    type=symbol
}

\newglossaryentry{rl:collisionRateAA}{
    name=\ensuremath{CR_{\text{A-A}}},
    description={Agent-Agent Collision Rate},
    sort={Agent-Agent Collision Rate},
    type=symbol
}

\newglossaryentry{rl:collisionRateAL}{
    name=\ensuremath{CR_{\text{A-L}}},
    description={Agent-Lanelet Collision Rate},
    sort={Agent-Lanelet Collision Rate},
    type=symbol
}

\newglossaryentry{rl:safetyRate}{
    name=\ensuremath{SR},
    description={Safety Rate},
    sort={Safety Rate},
    type=symbol
}

\newglossaryentry{rl:safetyRateTotal}{
    name=\ensuremath{SR_{\text{total}}},
    description={Total Safety Rate},
    sort={Total Safety Rate},
    type=symbol
}

\newglossaryentry{rl:safetyRateAA}{
    name=\ensuremath{SR_{\text{A-A}}},
    description={Agent-Agent Safety Rate},
    sort={Agent-Agent Safety Rate},
    type=symbol
}

\newglossaryentry{rl:safetyRateAL}{
    name=\ensuremath{SR_{\text{A-L}}},
    description={Agent-Lanelet Safety Rate},
    sort={Agent-Lanelet Safety Rate},
    type=symbol
}

\newglossaryentry{rl:centerLineDeviation}{
    name=\ensuremath{CD},
    description={Center Line Deviation},
    sort={Score: Center Line Deviation},
    type=symbol
}

\newglossaryentry{rl:laneAdherence}{
    name=\ensuremath{LA},
    description={Lane Adherence},
    sort={Score: Lane Adherence},
    type=symbol
}

\newglossaryentry{rl:averageSpeed}{
    name=\ensuremath{AS},
    description={Average Speed},
    sort={Average Speed},
    type=symbol
}

\newglossaryentry{rl:baseProblem}{
    name=\ensuremath{\mathcal{G}_{\baseSubscript}},
    description={Base problem},
    sort={Base problem},
    type=symbol
}

\newglossaryentry{rl:priProblem}{
    name=\ensuremath{\mathcal{G}_{\priSubscript}},
    description={Priority assignment problem},
    sort={Priority assignment problem},
    type=symbol
}

\newglossaryentry{rl:priRank}{
    name=\ensuremath{\mathcal{R}_{\priSubscript}},
    description={Priority rank},
    sort={Priority rank},
    type=symbol
}

\newglossaryentry{rl:higherPriorities}{
    name=\ensuremath{\leftarrow},
    description={Higher priorities},
    sort={Higher priorities},
    type=symbol
}

\DeclareAcronym{ap}{
    short = AP,
    long  = Action Propagation,
}
\DeclareAcronym{c2c}{
    short = C2C,
    long  = Center-to-Center,
}
\DeclareAcronym{cav}{
    short = CAV,
    long  = Connected and Automated Vehicle,
}
\DeclareAcronym{cbf}{
    short = CBF,
    long  = Control Barrier Function,
}
\DeclareAcronym{cg}{
    short = CG,
    long = Center of Gravity,
    short-plural = s,
    long-plural-form = Centers of Gravity,
}
\DeclareAcronym{cnn}{
    short = CNN,
    long  = Convolutional Neural Network
}
\DeclareAcronym{cpm}{
    short = CPM,
    long  = Cyber-Physical Mobility
}
\DeclareAcronym{cpmlab}{
    short = CPM Lab,
    long  = Cyber-Physical Mobility Lab
}
\DeclareAcronym{dmpc}{
    short = DMPC,
    long  = distributed model predictive control
}
\DeclareAcronym{dql}{
    short = DQL,
    long  = Deep Q-Learning
}

\DeclareAcronym{hocbf}{
    short = HOCBF,
    long  = High-Order \ac{cbf},
    short-indefinite = an,
}
\DeclareAcronym{il}{
    short = IL,
    long  = Imitation Learning,
    short-indefinite = an,
}
\DeclareAcronym{mappo}{
    short = MAPPO,
    long  = Multi-Agent \ac{ppo},
    short-indefinite = an,
}
\DeclareAcronym{maddpg}{
    short = MADDPG,
    long  = Multi-Agent Deep Deterministic Policy Gradient,
    short-indefinite = an,
}
\DeclareAcronym{mas}{
    short = MAS,
    long  = Multi-Agent System,
    short-indefinite = an,
}
\DeclareAcronym{mdp}{
    short = MDP,
    long  = Markov decision process,
    short-indefinite = an,
}
\DeclareAcronym{mg}{
    short = MG,
    long  = Markov Game,
    short-indefinite = an,
}
\DeclareAcronym{ml}{
    short = ML,
    long  = Machine Learning,
    short-indefinite = an,
}
\DeclareAcronym{mtv}{
    short = MTV,
    long  = Minimum Translation Vector,
    short-indefinite = an,
}
\DeclareAcronym{mpc}{
    short = MPC,
    long  = model predictive control,
    short-indefinite = an,
}
\DeclareAcronym{marl}{
    short = MARL,
    long  = Multi-Agent Reinforcement Learning,
    short-indefinite = an,
}
\DeclareAcronym{ocp}{
    short = OCP,
    long  = Optimal Control Problem,
    short-indefinite = an,
    long-indefinite = an,
}
\DeclareAcronym{per}{
    short = PER,
    long  = Prioritized Experience Replay
}
\DeclareAcronym{pomdp}{
    short = POMDP,
    long  = Partially Observable \ac{mdp}
}
\DeclareAcronym{pomg}{
    short = POMG,
    long  = Partially Observable \ac{mg}
}
\DeclareAcronym{ppo}{
    short = PPO,
    long  = Proximal Policy Optimization
}
\DeclareAcronym{qp}{
    short = QP,
    long  = Quadratic Program,
}
\DeclareAcronym{rhc}{
    short = RHC,
    long  = receding horizon control,
    short-indefinite = an,
}
\DeclareAcronym{rl}{
    short = RL,
    long  = Reinforcement Learning,
    short-indefinite = an,
}
\DeclareAcronym{sat}{
    short = SAT,
    long = Separating Axis Theorem,
    short-indefinite = an
}
\DeclareAcronym{som}{
    short = SOM,
    long  = Self Other-Modeling,
    short-indefinite = a,
}
\DeclareAcronym{zsg}{
    short = ZSG,
    long  = Zero-Shot Generalization,
}
\newglossaryentry{def:agent}{
	name=agent,
	description={An agent is a system which is composed of at least one of the three elements: sensors, actuators, and a dynamic behavior.%
    },
}

\newglossaryentry{def:agentActive}{
	name=active agent,
	description={Active \glspl{def:agent} are \glspl{def:agent} which are connected using a communication
    network over which they can exchange data. The exchanged data is
    used by the \glspl{def:agent}' controllers to find appropriate inputs to achieve their
    goals while interacting with other \glspl{def:agent}.
    Additionally, active \glspl{def:agent} consider \glspl{def:agentPassive}},
    parent=def:agent,
}

\newglossaryentry{def:agentPassive}{
	name=passive agent,
	description={Passive \glspl{def:agent} are \glspl{def:agent} without networked control. However, they can communicate their data like current and future states to \glspl{def:agentActive}, or they can be detected by \glspl{def:agentActive}' sensors.%
    },
    parent=def:agent,
}

\newglossaryentry{def:distrutedSolutionQuality}{
	name=distributed solution quality,
	description={%
        The quality $q\in [0,1]$ of the solution in \ac{dmpc} is the networked objective function value ${\fcnObjective}_{c}$ for the solution of the corresponding \ac{cmpc} formulation divided by the objective function value ${\fcnObjective_d}$ for the solution of the \ac{dmpc} formulation
        \begin{equation}
            q = \frac{\fcnObjective_c}{\fcnObjective_d}.
        \end{equation}
    },
}

\newglossaryentry{def:mas}{
	name=multi-agent system,
	description={A system consisting of multiple \glspl{def:agent}.%
    },
}

\newglossaryentry{def:ncs}{
	name=networked control system,
	description={A system consisting of multiple \glspl{def:agentActive}.%
    },
}

\newglossaryentry{def:prediction}{
	name=prediction,
	description={
        A prediction $\agentPrediction^{\anAgent}_{\cdot\vert \timestep}$ of \gls{def:agent} $\anAgent$ is its predicted state trajectory as obtained from the solution of its \ac{ocp} at time $\timestep$.
        A prediction $\agentPredictionForAgentAFromAgentBAtTimeC{\anAgent}{\anotherAgent}{\timestep}$ of \gls{def:agent} $\anAgent$ for \gls{def:agent} $\anotherAgent$ is agent $\anotherAgent$'s state trajectory as viewed from agent $\anAgent$ at time $\timestep$. It is obtained by communication or by predicting \gls{def:agent} $\anotherAgent$'s state trajectory with its model using the solution to its \ac{ocp}.%
    },
}

\newglossaryentry{def:consistency}{
	name=prediction consistency,
	description={%
        \Iac{ncs} is prediction consistent at time step $\timestep$ if the \gls{def:prediction} \agentPredictionForAgentAFromAgentBAtTimeC{\anotherAgent}{\anAgent}{\timestep} of every agent $\anAgent\in\setAgents$ for each of its neighbors $\anotherAgent \in \setNeighbors{\anAgent}$ equals the actual \gls{def:prediction} $\agentPrediction^{(\anotherAgent)}_{\cdot\vert \timestep}$ of its neighbors, i.e.,
        \begin{equation}
            \agentPredictionForAgentAFromAgentBAtTimeC{\anotherAgent}{\anAgent}{\timestep}=\agentPrediction^{(\anotherAgent)}_{\cdot\vert \timestep}, \quad \forall \anAgent \in \setAgents, \forall \anotherAgent \in \setNeighbors{\anAgent}.
        \end{equation}%
    }
}

\newglossaryentry{def:ncsFeasible}{
	name=NCS-feasible,
	description={%Eigenschaft aller Loesungen
        The solutions $\sysControlInputs_{\cdot \vert \timestep}\ofAgent{\anAgent}$ of all agents $i\in\setAgents$ are \acs*{ncs}-feasible if the stacked solution vector $\bm{U}_{\cdot \vert \timestep} = \left( \sysControlInputs_{\cdot \vert \timestep}\ofAgent{1}, \ldots, \sysControlInputs_{\cdot \vert \timestep}\ofAgent{\numAgents} \right)\transposed$ satisfies all constraints of the corresponding central \acf*{ocp} considering all agents.%
    },
}

\newglossaryentry{def:feasibleAgent}{
	name=agent-feasible,
	description={%
        A solution is agent-feasible if it satisfies the constraints of to the corresponding agent's \ac{ocp}.%
    },
}

\newglossaryentry{def:networkedObjectiveFunction}{
	name=networked objective function,
	description={%
        The objective function value ${\fcnObjective}$ in \iac{ncs} formulation is the sum of all agent objective functions \fcnObjective[i]
        \begin{equation}
            \fcnObjective = \sum_{i}^{i\in\setAgents} \fcnObjective[i].
        \end{equation}
    },
}

\newglossaryentry{def:optimalPriorityAssignment}{
	name=optimal priority assignment,
	description={%
        The optimal priority assignment results in a feasible solution for every agent with the lowest networked objective function value.%
    },
}

\newglossaryentry{def:graph}{
	name=graph,
	description={%
        A directed graph $\graphDirected = \left(\setVertices,\setEdges\right)$ is a pair of two sets,
        the set of vertices $\setVertices=\set{1,\dots,\numAgents}$
        and the set of directed edges $\setEdges \subseteq \setVertices \times \setVertices$.
        The edge from $i$ to $j$ is denoted by $\edgeDirected{i}{j}$.
        An undirected graph $\graphUndirected = \left(\setVertices,\setEdges\right)$ is a special form of a directed graph in which every edge is directed both ways, i.e., $\edgeDirected{i}{j} \in \setEdges \iff \edgeDirected{j}{i} \in \setEdges$.
    },
}

\newglossaryentry{def:path}{
	name=path,
	description={%
        A path of a graph $\graphDirected$ is a subgraph $\graphDirected_{\graphPath} = \left(\setVertices_{\graphPath},\setEdges_{\graphPath}\right)\subseteq\graphDirected$ with distinct vertices $\setVertices_{\graphPath}=\{i_{1},i_{2},i_{3},\ldots,i_{\numVerticesPath-1},i_{\numVerticesPath}\}$ and distinct edges $\setEdges_{\graphPath}=\{\edgeDirected{i_{1}}{i_{2}},\edgeDirected{i_{2}}{i_{3}},\ldots,\edgeDirected{i_{\numVerticesPath-1}}{i_{\numVerticesPath}}\}$, with $\numVerticesPath$ being the number of vertices of the path. The length of the path is defined as $\numVerticesPath-1$.
    },
}

\newglossaryentry{def:graph:adjacency}{
	name=adjacency,
	description={%
    A vertex $j$ is a predecessor of vertex $i$ iff $\edgeDirected{j}{i}\in\setEdges$.
    The set of predecessors of vertex $i$ is denoted by
    \begin{equation}
        \setPredecessors{i}=\set{j \mid \edgeDirected{j}{i}\in\setEdges}.
    \end{equation}
    A vertex $j$ is a successor of vertex $i$ iff $\edgeDirected{i}{j}\in\setEdges$.
    The set of successors of vertex $i$ is denoted by
    \begin{equation}
        \setSuccessors{i}=\set{j \mid \edgeDirected{i}{j}\in\setEdges}.
    \end{equation}
    A vertex $j$ is a neighbor to or adjacent to vertex $i$ if it is either a predecessor or a successor.
    The set of neighbors of vertex $i$ is denoted by
    \begin{equation}
        \setNeighbors{i}= \setSuccessors{i} \cup \setPredecessors{i}.
    \end{equation}
    },
    parent=def:graph,
}

\newglossaryentry{def:graph:degree}{
	name=degree,
	description={%
        The degree $\vertexDegree{i} = \lvert \setNeighbors{i} \rvert$ denotes the number of the adjacent vertices of vertex $i$. 
        The number of incoming edges called in-degree is denoted by $\vertexInDegree{i} = \lvert \setPredecessors{i} \rvert$.
        The number of outgoing edges called out-degree is denoted by $\vertexOutDegree{i} = \lvert \setSuccessors{i} \rvert$.%
    },
    parent=def:graph,
}

\newglossaryentry{def:couplingGraph}{
	name=coupling graph,
	description={A coupling graph $\graphDirected=(\setVertices,\setEdges)$ is a graph that represents the interaction between agents. Vertices represent agents and edges denote coupling objectives or constraints. A vertex from agent $i$ to agent $j$ corresponds to a coupling objective or constraint in the \ac{ocp} of agent $j$.%
    },
}

\newglossaryentry{def:matrix:Adjacency}{
	name=adjacency matrix,
	description={An adjacency matrix represents a graph with $\numAgents$ vertices in a matrix $\matAdjacency \in \set{0,1}^{\numAgents\times\numAgents}$ with entries
    \begin{equation}
        \matAdjacencyElement{ij} =
            \begin{cases}
                1 & \text{ if } \edgeDirected{i}{j} \in \setEdges \\
                0 & \text{ otherwise.}
            \end{cases}
    \end{equation}
    },
}

\newglossaryentry{def:tCompNcs}{
	name=computation time of \iac{ncs},
	description={%
        The computation time $\tCompNcs$ of \iac{ncs} is the time required for the \ac{ncs} to measure the states, formulate and solve the \ac{ocp}, apply the inputs to all agents, and communicate the required data. 
    },
}

\newglossaryentry{def:setReachable}{
	name=reachable set,
	description={%
        The reachable set of states $\setReachable$ of an agent from an initial time $t_{\text{init.}}$ to an end time $t_{\text{end}}$ is
            \begin{equation}\label{eq:setReachable}
                \setReachable_{[t_{\text{init.}},t_{\text{end}}] \mid t_{\text{init.}}} = \biggl\{ \int_{t_{\text{init.}}}^{t_{\text{end}}} \sysModelContinuous(\sysState,\sysControlInputs)dt
                \biggm| \sysState(t_{\text{init.}}) \in \setFeasibleStates(t_{\text{init.}}), \forall t: \sysControlInputs \in \setFeasibleInputs \biggr\},
            \end{equation}
        with the possible system initial states $\sysState(t_{\text{init.}})$ being bounded by its initially admissible set $\setFeasibleStates(t_{\text{init.}}) \subseteq \setRealNumbers^{\numStates}$, and the possible system control inputs $\sysControlInputs$ being bounded by its admissible set $\setFeasibleInputs \subseteq \setRealNumbers^{\numInputs}$.
    },
}

\newglossaryentry{def:conflictualDecisions}{
	name=conflictual decisions in \iac{ncs},
	description={%
        Consider two decisions made by two agents of \iac{ncs} at time step $\timestep$ with a duration $N_k$.
        They are deemed conflictual if the predicted outcome of the decisions violates the \ac{ncs}-feasibility at some point in time.%
    },
}

\newglossaryentry{def:conflictualSpace}{
	name=conflictual space of \iac{ncs},
	description={%
        In dynamic systems, the state space represents the set of all possible states the systems can occupy. 
        The conflictual space refers to a subset, or potentially the entirety, of this state space where whether decisions are conflictual is determined.
    },
}

\input{00_copyright.tex}

\def\BibTeX{{\rm B\kern-.05em{\sc i\kern-.025em b}\kern-.08em
  T\kern-.1667em\lower.7ex\hbox{E}\kern-.125emX}}

\begin{document}
\title{\LARGE \bf
    A Learning-Based Control Barrier Function for Car-Like Robots: Toward Less Conservative Collision Avoidance
    % Options:
    % Towards Safety-Guaranteed, Less Conservative Control Barrier Functions for Car-Like Robots via Heading-Aware Safety Margins
    % A Heading-Aware Safety Margin for Reducing Conservatism in Control Barrier Functions with Provably Guaranteed Safety
    % Directional Safety: Enhancing Control Barrier Functions with Heading-Aware Safety Margin in Car-Like Robots
    \thanks{This research was supported by the Bundesministerium für Digitales und Verkehr (German Federal Ministry for Digital and Transport) within the project ``Harmonizing Mobility'' (grant number 19FS2035A).}
}

\author{
    Jianye Xu$^{1}$\,\orcidlink{0009-0001-0150-2147},~\IEEEmembership{Student~Member,~IEEE},
    Bassam Alrifaee$^{2}$\,\orcidlink{0000-0002-5982-021X},~\IEEEmembership{Senior Member, ~IEEE}% <-this % stops a space
    \thanks{$^{1}$Department of Computer Science, RWTH Aachen University, Germany, \texttt{xu@embedded.rwth-aachen.de}}
    \thanks{$^{2}$Department of Aerospace Engineering, University of the Bundeswehr Munich, Germany, \texttt{bassam.alrifaee@unibw.de}}
}
    \maketitle
\setpagestyle % This will automatically adapt based on \preprinttrue or \preprintfalse. Must be executed after \maketitle

% ! Abstract
\begin{abstract}
\noindent
We propose a learning-based \ac{cbf} to reduce conservatism in collision avoidance for car-like robots. Traditional \acp{cbf} often use the Euclidean distance between robots' centers as a safety margin, which neglects their headings and approximates their geometries as circles. Although this simplification meets the smoothness and differentiability requirements of \acp{cbf}, it may result in overly conservative behavior in dense environments. We address this by designing a safety margin that considers both the robot's heading and actual shape, thereby enabling a more precise estimation of safe regions. Because this safety margin is non-differentiable, we approximate it with a neural network to ensure differentiability. In addition, we propose a notion of relative dynamics that makes the learning process tractable. In a case study, we establish the theoretical foundation for applying this notion to a nonlinear kinematic bicycle model. Numerical experiments in overtaking and bypassing scenarios show that our approach reduces conservatism (e.g., requiring \SI{33.5}{\percent} less lateral space for bypassing) without incurring significant extra computation time.
\par\medskip
\noindent
Code: \href{https://github.com/bassamlab/sigmarl}{\small github.com/bassamlab/sigmarl}
\end{abstract}

% \begin{keywords}
% Control barrier function, collision avoidance, car-like robots, safety guarantee
% \end{keywords}
\acresetall  % reset acronyms used in abstract

\section{Introduction}\label{sec:introduction}
Collision avoidance for car-like robots often involves non-convex optimization problems \cite{zhang2021optimizationbased}. \Acp{cbf} provide a tool to replace non-convex safety constraints with affine constraints in the control input \cite{ames2017control}. However, standard \ac{cbf} formulations tend to over-approximate the actual geometry of car-like robots because they require a continuous and differentiable function \cite{ames2014control}. A common simplification is to approximate the robots as circles, which simplifies distance computation but ignores the robots' actual shapes and headings \cite{chen2017obstacle, wang2017safety, xiao2019control, xu2025highorder, xiao2023barriernet, zhang2023spatialtemporalaware, han2024multiagent, chen2020guaranteed, singletary2021comparative, gao2023online, han2025riskaware}. While this ensures the continuity and differentiability required by \acp{cbf}, it can lead to overly conservative behaviors, especially in dense environments. \Cref{fig_over_approximation} illustrates an example of conservatism caused by the circle approximation, where vehicle $i$ is prevented from overtaking vehicle $j$. Our work aims to address this limitation by incorporating the actual geometries and headings of car-like robots into the safety constraints, thereby enabling less conservative collision avoidance.

\begin{figure}[t]
    \centering
    \includegraphics[width=0.8\linewidth]{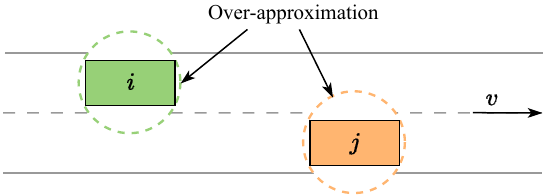}
    \caption{An example of conservatism caused by the circle approximation. Vehicle $i$ is prevented from overtaking $j$.}
    \label{fig_over_approximation}
\end{figure}

\subsection{Related Work}\label{sec:related}
Collision avoidance for car-like robots is a well-studied problem. Although optimization-based approaches like model predictive control \cite{ji2016path, alrifaee2017networkeda, scheffe2023recedinga, scheffe2023sequentiala} are widely used, they can be computationally intensive. \Acp{cbf} offer a way to enforce safety constraints in a computationally efficient manner \cite{ames2014control}.

Traditional \acp{cbf} for collision avoidance typically simplify the robot geometries as circles and define safety margins based on the Euclidean distance between the circles' centers, referred to as the \ac{c2c}-based safety margin. In \cite{chen2017obstacle}, \acp{cbf} were applied to ensure a safe distance from an avoidable set that establishes safe boundaries around round-shaped moving obstacles. In \cite{wang2017safety}, safety barrier certificates were introduced for collision avoidance in multi-robot systems, where each robot is approximated as a circle. In \cite{xiao2019control}, \acp{cbf} were extended to constraints with high relative degrees, validated using circular robots. In our previous work \cite{xu2025highorder}, a method called Truncated Taylor \ac{cbf} was proposed to simplify control design for constraints with high relative degrees, with numerical experiments conducted on circular robots. In \cite{xiao2023barriernet}, off-center disks were used to reduce the conservatism inherent in the \ac{c2c}-based safety margin. In the context of \acp{cav}, \cite{zhang2023spatialtemporalaware} and \cite{han2024multiagent} applied a \ac{c2c}-based safety margin within a multi-agent \ac{rl} framework to ensure the safety of the learned policies. Other works employing the \ac{c2c}-based safety margin include \cite{chen2020guaranteed, singletary2021comparative, gao2023online, han2025riskaware}. Although approximating robots as circles simplifies computations and maintains the differentiability required by \acp{cbf}, this simplification does not capture the actual shapes and orientations of car-like robots. Such approximation can lead to conservative behaviors, restricting their ability to navigate efficiently in complex environments.

To improve upon the circle approximation, some researchers approximate robots or obstacles as ellipses (or ellipsoids in the case of a 3D space), which better represent their elongated shapes, although computing their distances is challenging. Early work in \cite{rimon1997obstacle} proposed a conservative distance estimate between ellipsoids, formulated as an eigenvalue problem. Study \cite{verginis2019closedform} derived a closed-form expression that represents a distance metric between two ellipsoids in a 3D space. In \cite{glotzbach2010advanceda}, an approach for trajectory replanning of unmanned marine vehicles was adopted in which obstacles are modeled as ellipses that are augmented to consider the vehicle width. In \cite{tanner2003nonholonomic}, robots and obstacles were represented by sets of ellipsoids, and a point-world transformation is proposed to convert these ellipsoids to points to simplify collision avoidance. Some works use a mixture of circles and ellipses for shape approximation by either approximating the ego robot with a circle and its surrounding robots with ellipses \cite{schwarting2018safe, jian2023dynamic} or vice versa \cite{liu2024improved}. These approaches reduce conservatism compared to a purely circle-based approximation but still do not fully capture the actual shape of car-like robots. Note that among these works, \acp{cbf} are only used in \cite{verginis2019closedform} and \cite{jian2023dynamic}. 

\subsection{Paper Contributions}\label{sec:contributions}
Our main contributions are threefold.
\begin{enumerate}
    \item We introduce \iac{mtv}-based, non-differentiable safety margin for collision avoidance of car-like robots that accounts for their actual geometries and headings. This safety margin enables a more precise estimation of safe regions and therefore reduces conservatism.
    \item We propose the notion of relative dynamics to enable tractable learning of this non-differentiable safety margin via a differentiable neural network, making it suitable as a candidate \ac{cbf}.
    \item We conduct a case study in which we establish the theoretical foundation for applying our proposed safety margin to car-like robots modeled by a nonlinear kinematic bicycle model.
\end{enumerate}
Although the concept of MTV is widely used in collision avoidance \cite{ericson2004real}, our work is the first to incorporate it into \acp{cbf} for collision avoidance of car-like robots.

\subsection{Notation}\label{sec:notation}
A variable $x$ is annotated with a superscript $x^i$ if it belongs to robot $i$. A relative variable includes two letters in its superscript to indicate the direction, e.g., $x^{ji}$ denotes $x$ of robot $j$ relative to that of robot $i$. If the relative variable is expressed in robot $i$'s ego perspective rather than in the global coordinate system, an underline is used, e.g., $x^{j\underline{i}}$. Vectors are generally denoted in bold. The dot product of two vectors $\bm{a}$ and $\bm{b}$ is represented by $\bm{a} \cdot \bm{b}$. Time arguments of time-varying variables are often omitted for brevity.

\subsection{Paper Structure}
\Cref{sec:preliminaries} revisits the kinematic bicycle model and \ac{cbf}. \Cref{sec:main} introduces our \ac{mtv}-based safety margin and its integration into \ac{cbf}. \Cref{sec:case-study} presents our case study and establishes a theoretical foundation for applying the proposed safety margin for collision avoidance of car-like robots modeled by the kinematic bicycle model. Finally, \Cref{sec:conclusions} concludes and outlines future research directions.

\section{Preliminaries}\label{sec:preliminaries}
We consider an input-affine control system given by
\begin{equation} \label{eq:general-dynamics}
    \dot{\bm{x}} = f(\bm{x}) + g(\bm{x})\bm{u}.
\end{equation}
Here, $f: \mathbb{R}^n \to \mathbb{R}^n$ and $g: \mathbb{R}^n \to \mathbb{R}^{n \times m}$ are Lipschitz continuous. $\bm{x} \in \mathcal{X} \subset \mathbb{R}^n$ is the state vector, and $\bm{u} \in \mathcal{U} \subset \mathbb{R}^m$ is the control input vector, where $n$ and $m$ denote the dimensions of the state space and control input space, and $\mathcal{X}$ and $\mathcal{U}$ are their respective admissible sets.

\subsection{Kinematic Bicycle Model} \label{sec:bicycle-model}
We employ the kinematic bicycle model to model the dynamics of car-like robots in our case study, as it captures the essential dynamics required for motion planning and control \cite{rajamani2011vehicle, polack2017kinematic}. This model approximates the robot as a single-track model with two wheels, as depicted in \cref{fig:bicycle-model}.

We define the state vector as $\bm{x} \coloneqq [x, y, \psi, v, \delta]^\top \in \mathbb{R}^5$, where $x$ and $y$ denote the position in the global coordinate system, and $\psi$, $v$, and $\delta$ denote the heading (or yaw angle), speed, and steering angle, respectively. The control input vector is defined as $\bm{u} \coloneqq [u_v, u_\delta]^\top \in \mathbb{R}^2$, where $u_v$ denotes the acceleration and $u_\delta$ the steering rate. The dynamics of the kinematic bicycle model are given by
\begin{equation} \label{eq:kinematic-bicycle-model}
    \dot{\bm{x}} =
    \begin{bmatrix}
        v \cos(\psi + \beta) \\
        v \sin(\psi + \beta) \\
        \dfrac{v}{\ell_{wb}} \tan(\delta) \cos(\beta) \\
        0 \\
        0
    \end{bmatrix}
    +
    \begin{bmatrix}
    0 & 0 \\
    0 & 0 \\
    0 & 0 \\
    1 & 0 \\
    0 & 1
    \end{bmatrix}
    \begin{bmatrix}
        u_v \\
        u_\delta
    \end{bmatrix},
\end{equation}
where $\ell_{wb} \in \mathbb{R}$ denotes the wheelbase of the vehicle and the slip angle $\beta \in \mathbb{R}$ is computed as
\begin{equation} \label{eq:beta}
    \beta = \tan^{-1} \left( \dfrac{\ell_r}{\ell_{wb}} \tan \delta \right),
\end{equation}
with $\ell_r \in \mathbb{R}$ representing the rear wheelbase.

\begin{figure}[t!]
    \centering
    \includegraphics[width=0.8\linewidth]{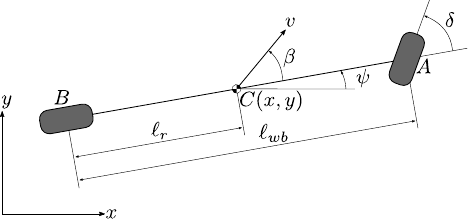}
    \caption{
        The kinematic bicycle model.
        $C$: center of gravity;
        $x, y$: $x$- and $y$-coordinates;
        $v$: velocity;
        $\beta$: slip angle;
        $\psi$: yaw angle;
        $\delta$: steering angle;
        $\ell_{wb}$: wheelbase;
        $\ell_{r}$: rear wheelbase.
    }
    \label{fig:bicycle-model}
\end{figure}

\subsection{Control Barrier Functions}
\Acp{cbf} provide a tool to enforce safety constraints by ensuring that the system's state remains within a safe set. Assume the safe set is defined as
\begin{equation} \label{eq:safe-set}
    C = \{ \bm{x} \in \mathcal{X} \mid h(\bm{x}) \geq 0 \},
\end{equation}
where $h: \mathcal{X} \rightarrow \mathbb{R}$ is a continuously differentiable function.

\begin{definition}[Class $\mathcal{K}$ Function]
A Lipschitz continuous function $\alpha: \mathbb{R} \rightarrow \mathbb{R}$ is said to belong to class $\mathcal{K}$ if it is strictly increasing and satisfies $\alpha(0) = 0$. 
\end{definition}

\begin{definition}[\acl{cbf} \cite{ames2014control, ames2017control}]\label{def:cbf} 
Given a set $C$ as in \eqref{eq:safe-set}, a continuously differentiable function $h: \mathcal{X} \rightarrow \mathbb{R}$ is a candidate \ac{cbf} for system \eqref{eq:general-dynamics} if there exists a class $\mathcal{K}$ function $\alpha$ such that 
\begin{equation*}
    \sup_{\bm{u} \in \mathcal{U}} \dot{h}(\bm{x}, \bm{u}) \geq 0, \quad \forall \bm{x} \in C.
\end{equation*}
Here, $\dot{h}(\bm{x}, \bm{u})$ denotes the time derivative of $h(\bm{x})$ and is given by $\frac{dh(\bm{x})}{dt} = \frac{\partial h(\bm{x})}{\partial \bm{x}} \dot{\bm{x}} = \frac{\partial h(\bm{x})}{\partial \bm{x}} \left(f(\bm{x}) + g(\bm{x}) \bm{u}\right)$, which is equivalent to the Lie derivative expression
$\frac{dh(\bm{x})}{dt} = L_f h(\bm{x}) + L_g h(\bm{x}) \bm{u}$, with $L_f$ and $L_g$ denoting the Lie derivatives along $f$ and $g$, respectively.
\end{definition}

\begin{definition}[Forward Invariant] \label{def:forward-invariant} 
A set $C \subset \mathcal{X}$ is forward invariant for system \eqref{eq:general-dynamics} if for any initial state $\bm{x}(t_0) \in C$, the solution satisfies $\bm{x}(t) \in C$ for all $t \geq t_0$. 
\end{definition}

In practice, \acp{cbf} often have a high \emph{relative degree}. The relative degree of a continuously differentiable function w.r.t. a system is the number of times the function must be differentiated along the system dynamics until the control input explicitly appears.

\begin{definition}[Relative Degree \cite{xiao2019control}] \label{def:relative-degree}
    A continuously differentiable function $ h: \mathcal{X} \to \mathbb{R} $ is said to have relative degree $ r \in \mathbb{N} $ w.r.t. system \eqref{eq:general-dynamics} if for all $\bm{x} \in \mathcal{X}$, $L_g L_f^{i} h(\bm{x}) = 0, \forall i \in \{0, 1, \ldots, r-2\}$, and $L_g L_f^{r-1} h(\bm{x}) \neq 0$.
\end{definition}

Let $h(\bm{x})$ be continuously differentiable. Define
\begin{equation} \label{eq:psi-0}
    \Psi_0(\bm{x}) \coloneqq h(\bm{x}),
\end{equation}
and recursively define 
\begin{equation}\label{eq:psi-r}
    \Psi_i(\bm{x}) \coloneqq \dot{\Psi}_{i-1}(\bm{x}) + \alpha_i\bigl(\Psi_{i-1}(\bm{x})\bigr), \quad i \in \{1,\ldots,r\},
\end{equation}
where $\alpha_i(\cdot)$ is an $(r-i)$th times differentiable class $\mathcal{K}$ function. Furthermore, define
\begin{equation} \label{eq:c-r}
    C_i \coloneqq \{\bm{x} \in \mathcal{X}: \Psi_{i-1}(\bm{x}) \geq 0\}, \quad i \in \{1,\ldots,r\}.
\end{equation}

\begin{definition}[High-Order \acp{cbf} \cite{xiao2019control}]\label{def:hocbf}
Given the sets $C_i$ as in \eqref{eq:c-r} and functions $\Psi_i$ as in \eqref{eq:psi-r} for all $i \in \{1,\ldots,r\}$, a function $h: \mathbb{R}^n \to \mathbb{R}$ is a candidate \ac{hocbf} with relative degree $r$ for system \eqref{eq:general-dynamics} if it is at least $r$th times differentiable and there exist $(r-i)$th times differentiable class $\mathcal{K}$ functions $\alpha_i$, for $i \in \{1,\ldots,r\}$, such that
\begin{equation} \label{eq:hocbf-constraint}
    \sup_{\bm{u} \in \mathcal{U}} \Psi_{r}(\bm{x}) \geq 0, \quad \forall \bm{x} \in \bigcap_{i=1}^{r} C_i.
\end{equation}
\end{definition}

\begin{theorem}[Thm. 4 in \cite{xiao2019control}]\label{theorem:hocbf-forward-invariant}
Given a candidate \ac{hocbf} $h(\bm{x})$ for system \eqref{eq:general-dynamics} as in \cref{def:hocbf}, if the initial state satisfies $\bm{x}(t_0) \in \bigcap_{i=1}^{r} C_i$, then any Lipschitz continuous controller that satisfies \eqref{eq:hocbf-constraint} renders the set $\bigcap_{i=1}^{r} C_i$ forward invariant for all $t \geq t_0$.
\end{theorem}

\section{MTV-Based Safety Margin as \acp{cbf}}\label{sec:main}
In this section, we introduce an MTV-based safety margin for collision avoidance of car-like robots that takes into account their actual geometries and headings. We also describe how to incorporate this safety margin into a \ac{cbf} framework.

\subsection{MTV-Based Safety Margin}
The \ac{sat} is a fundamental concept in computational geometry for detecting collisions between convex shapes \cite{gottschalk1996obbtree}. It states that two convex shapes do not intersect if there exists an axis along which their projections do not overlap. If no such axis exists, the shapes are colliding. In this context, the MTV represents the smallest vector required to separate the shapes \cite{ericson2004real}. We extend the concept of the MTV to define a safety margin for two car-like robots, which are approximated as rectangles. This safety margin represents the minimal movement required for one robot to make contact with the other, taking into account their geometries and headings.

\Cref{alg:mtv} details the computation of the MTV-based safety margin for a given pair of rectangles $i$ and $j$. For each rectangle $k \in \{ i, j \}$ and for each of its orthogonal axes $a \in \{ x_A^k, y_A^k \}$, the algorithm projects the vertices of both rectangles onto axis $a$. If the projections do not overlap on axis $a$, the gap $g_a$ is positive and represents the separation; otherwise, $g_a$ is negative and corresponds to the overlapping length. Based on the signs of $g_{x_A^k}$ and $g_{y_A^k}$, the algorithm classifies the relative position as mutually separating, overlapping along one axis, or non-separating, and then computes a distance metric $d^k$ accordingly (lines \ref{line:if-overlap-begin} to \ref{line:if-overlap-end}). Finally, the safety margin $d_{\text{MTV}}$ is determined based on $d^i$ and $d^j$ (lines \ref{line:if-separate-begin} to \ref{line:if-separate-end}). \Cref{fig_mtv} illustrates a case where $g_{x_A^i} < 0$, $g_{y_A^i} < 0$, $g_{x_A^j} > 0$, and $g_{y_A^j} < 0$, resulting in $d^i = -\min(|g_{x_A^i}|, |g_{y_A^i}|) < 0$, $d^j = g_{y_A^j} > 0$, and hence $d_{\text{MTV}} = \max(d^i, d^j) = d^j$.

\renewcommand{\Comment}[2][.5\linewidth]{%
  \hfill{\raggedright\scriptsize\texttt{\textcolor{black}{$\triangleright$ #2}}}}
  
\begin{algorithm}[t] 
\caption{MTV-Based Safety Margin}
\label{alg:mtv} 
\begin{algorithmic}[1] 
\Input Positions, headings, widths, and lengths of rectangles $i$ and $j$ 
\Output Safety margin $d_{\text{MTV}}$ 

\State Compute the positions of both rectangles' vertices.
\For{each rectangle $k \in \{i, j\}$} 
    \State Compute the pair of orthogonal axes $\{x_A^k, y_A^k\}$ of $k$.

    \For{each axis $a \in \{x_A^k, y_A^k\}$}
        \State Project both rectangles' vertices onto axis $a$. 

        \If{the projections do not overlap on axis $a$} 
            \State $g_a \gets$ gap between the projections (positive).
        \Else 
            \State $g_a \gets$ overlapping length (negative).
        \EndIf
    \EndFor 

    \If{both $g_{x_A^k}, g_{y_A^k} > 0$} \Comment{Mutually separating} \label{line:if-overlap-begin}
        \State $d^k = \sqrt{(g_{x_A^k})^2 + (g_{y_A^k})^2}$.
    \ElsIf{both $g_{x_A^k}, g_{y_A^k} < 0$} \Comment{Separating by one rectangle's axes}
        \State $d^k = -\min(|g_{x_A^k}|, |g_{y_A^k}|)$. 
    \Else \Comment{Non-separating}
        \State $d^k = \max(g_{x_A^k}, g_{y_A^k})$.  
    \EndIf \label{line:if-overlap-end}
\EndFor 
% At this point, distances $d^i$ and $d^j$ have been computed
\If{both $d^i, d^j > 0$} \label{line:if-separate-begin}
    \Comment{$i$ and $j$ are separated by both their axes} 
    \State $d_{\text{MTV}} = \min(d^i, d^j)$. 
\ElsIf{both $d^i, d^j < 0$} 
    \Comment{$i$ and $j$ are unseparated}
    \State $d_{\text{MTV}} = -\min(|d^i|, |d^j|)$.
\Else 
    \Comment{One rectangle overlaps the other along its axes} 
    \State $d_{\text{MTV}} = \max(d^i, d^j)$.
\EndIf \label{line:if-separate-end}

\State \Return $d_{\text{MTV}}$ 
\end{algorithmic} 
\end{algorithm}

\begin{figure}[t]
    \centering
    \includegraphics[width=0.5\linewidth]{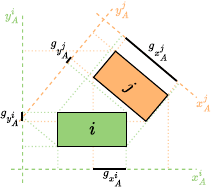}
    \caption{An example illustrating \cref{alg:mtv} for the case where $g_{x_A^i} < 0$, $g_{y_A^i} < 0$, $g_{x_A^j} > 0$, and $g_{y_A^j} < 0$.}
    \label{fig_mtv}
\end{figure}

\subsection{Learning-Based Safety-Margin Approximator}\label{sec:approximator}
To enable our MTV-based safety margin as a valid candidate \ac{cbf}, we approximate it using a neural network to obtain a differentiable safety-margin approximator. 

\Cref{alg:mtv} requires the positions, headings, widths, and lengths of the two rectangles, which results in a ten-dimensional input space. This high dimensionality can make learning intractable. To reduce the input space, we introduce the concept of \textit{relative dynamics} by using the relative position and relative heading of one rectangle from the perspective of the other. We denote the relative dynamics as
\begin{equation} \label{eq:relative-dynamics}
    \dot{\bm{x}}^{j\underline{i}} = f^{j\underline{i}}(\bm{x}^{j\underline{i}}) + g^{j\underline{i}}(\bm{x}^{j\underline{i}}) \bm{u},
\end{equation}
where $f^{j\underline{i}}: \mathbb{R}^3 \to \mathbb{R}^3$ and $g^{j\underline{i}}: \mathbb{R}^3 \to \mathbb{R}^{3 \times 4}$ are Lipschitz continuous. The relative state is defined as $\bm{x}^{j\underline{i}} \coloneqq [x^{j\underline{i}}, y^{j\underline{i}}, \psi^{j\underline{i}}]^\top \in \mathcal{X} \subset \mathbb{R}^3$, where $(x^{j\underline{i}}, y^{j\underline{i}})$ denote the relative position and $\psi^{j\underline{i}}$ the relative heading of robot $j$ w.r.t. robot $i$, as observed from $i$'s perspective. The joint control input is $\bm{u} \coloneqq \begin{bmatrix} (\bm{u}^i)^\top, (\bm{u}^j)^\top \end{bmatrix}^\top \in \mathcal{U} \subset \mathbb{R}^4$, with $\bm{u}^i \in \mathbb{R}^2$ and $\bm{u}^j \in \mathbb{R}^2$ being the control inputs for robots $i$ and $j$, respectively. We train a neural network $h_\theta(\bm{x}^{j\underline{i}}): \mathbb{R}^3 \to \mathbb{R}$, with parameters $\theta$, to approximate the function computed by \cref{alg:mtv}. Note that the width and length of the rectangles are not included in the network input because they are time-invariant, and the network can implicitly learn the geometric relationships.

\subsection{Construction of the Control Barrier Function}\label{sec:cbf-construction}
We employ the MTV-based safety-margin approximator $h_\theta(\bm{x}^{j\underline{i}})$ as a candidate \ac{cbf}. To account for the approximation error between $h_\theta$ and the true value computed by \cref{alg:mtv}, we consider an upper bound on the approximation error, $e_{\max} > 0$. We construct the \ac{cbf} as
\begin{equation} \label{eq:h}
h_{\mathrm{MTV}}(\bm{x}^{j\underline{i}}) = h_\theta(\bm{x}^{j\underline{i}}) - e_{\max}.
\end{equation}
Although considering the upper bound of the error can introduce conservatism, this is acceptable if it is sufficiently small.

\begin{assumption} \label{ass:error-upper-bound}
The upper bound of the approximation error $e_{\max}$ in \eqref{eq:h} is known.
\end{assumption}

We note that \cref{ass:error-upper-bound} is reasonable because the input space of the neural network can be restricted to a known range. This limitation prevents out-of-distribution issues and allows a thorough evaluation of all possible inputs. The three-dimensional input space comprises the relative position $(x^{j\underline{i}}, y^{j\underline{i}})$ and the relative heading $\psi^{j\underline{i}}$, where the heading is naturally bounded within $[-\pi,\pi]$. We limit the relative position to a small yet practical range around robot $i$. If robot $j$ is outside this range, we revert to the \ac{c2c}-based safety margin, for which the conservatism has a negligible impact.

\begin{assumption} \label{ass:h-differentiable}
The function $h_{\mathrm{MTV}}(\bm{x}^{j\underline{i}})$ in \eqref{eq:h} is at least $r$th times differentiable, where $r$ is its relative degree w.r.t. system \eqref{eq:general-dynamics}.
\end{assumption}

This assumption is justified since, according to \cref{def:hocbf}, a candidate \ac{hocbf} with relative degree $r$ must be at least $r$th times differentiable. Continuous differentiability can be ensured by using a continuously differentiable activation function, such as the \texttt{TanH} function, in the neural network.

\begin{theorem} \label{theorem:collision-free}
Let $r$ denote the relative degree of $h_{\mathrm{MTV}}$ in \eqref{eq:h} w.r.t. system \eqref{eq:general-dynamics}. Redefine $\Psi_0$ as in \eqref{eq:psi-0}, $\Psi_i$ as in \eqref{eq:psi-r}, and $C_i$ as in \eqref{eq:c-r} using $h_{\mathrm{MTV}}$, for all $i \in \{ 1, \ldots, r \}$. If the initial state satisfies $\bm{x}^{j\underline{i}}(t_0) \in \bigcap_{i=1}^{r} C_i$, then any Lipschitz continuous controller that satisfies
\begin{equation} \label{eq:k-cbf}
    \sup_{\bm{u} \in \mathcal{U}} \Psi_r\bigl(\bm{x}^{j\underline{i}}(t)\bigr) \geq 0, \quad \forall \bm{x}^{j\underline{i}}(t) \in \bigcap_{i=1}^{r} C_i,
\end{equation}
renders system \eqref{eq:relative-dynamics} collision-free for all $t \geq t_0$.
\end{theorem}

\begin{proof}
Since $h_{\mathrm{MTV}}$ is at least $r$th times differentiable, it qualifies as a candidate \ac{hocbf} with relative degree $r$. By \cite[Thm. 4]{xiao2019control}, the set $C_1$ is forward invariant for system \eqref{eq:relative-dynamics}. This implies that the state $\bm{x}^{j\underline{i}}(t)$ remains in $C_1$, i.e., $h_{\mathrm{MTV}}(\bm{x}^{j\underline{i}}) \geq 0$ for all $t \geq t_0$. With the ``conservative'' consideration of the upper bound  $e_{\max}$ of the approximation error in \eqref{eq:h}, the safety margin between robots $i$ and $j$ remains non-negative, ensuring that system \eqref{eq:relative-dynamics} remains collision-free for all $t \geq t_0$.
\end{proof}

\begin{remark}
Although the system defined in \eqref{eq:relative-dynamics} involves only two robots, \Cref{theorem:collision-free} can be extended to systems with any number of robots. For each pair of robots, we can construct a sub-system as in \eqref{eq:relative-dynamics}, resulting in $\binom{k}{2} = k(k-1)/2$ sub-systems for $k$ robots. The overall system remains collision-free as long as all sub-systems are collision-free. 
\end{remark}

\section{Case Study} \label{sec:case-study}
In this section, we conduct a case study to numerically evaluate the proposed MTV-based safety margin in simulations with two car-like robots. Without loss of generality, we employ the kinematic bicycle model \cite{rajamani2011vehicle} to model the robots' dynamics. We design two scenarios---an overtaking scenario and a bypassing scenario---to compare our MTV-based safety margin with the traditional \ac{c2c}-based safety margin. Codes for reproducing our experimental results and video demonstrations are available in our open-source repository\footnote{\href{https://github.com/bassamlab/sigmarl}{\small github.com/bassamlab/sigmarl}}.

First, we establish the theoretical foundation for applying our proposed notion of relative dynamics to the kinematic bicycle model in \Cref{sec:case-study-relative-dynamics}. Then, we describe how to train the MTV-based safety margin approximator and formulate the \ac{ocp} in \Cref{sec:ocp}. \Cref{sec:scenario-1} presents the experimental results for the overtaking scenario, and \Cref{sec:scenario-2} those for the bypassing scenario.

\subsection{Relative Dynamics of the Kinematic Bicycle Model} \label{sec:case-study-relative-dynamics}
We derive the explicit form of the relative dynamics \eqref{eq:relative-dynamics} for two car-like robots $i$ and $j$ modeled by the kinematic bicycle model \eqref{eq:kinematic-bicycle-model}. Without loss of generality, we designate robot $i$ as the ego robot and express the state of robot $j$ relative to $i$. The relative state in the global coordinate system is
\begin{equation} \label{eq:relative-state-global}
    \bm{x}^{ji} = \bm{x}^j - \bm{x}^i,
\end{equation}
where $\bm{x}^i \coloneqq \begin{bmatrix} x^i,\, y^i,\, \psi^i \end{bmatrix}^\top$ and $\bm{x}^j \coloneqq \begin{bmatrix} x^j,\, y^j,\, \psi^j \end{bmatrix}^\top$. We project $\bm{x}^{ji}$ into robot $i$'s ego coordinate system, yielding
\begin{equation} \label{eq:x-ego}
    \bm{x}^{j\underline{i}} =
    \begin{bmatrix}
        x^{j\underline{i}} \\
        y^{j\underline{i}} \\
        \psi^{j\underline{i}}
    \end{bmatrix}
    =
    \begin{bmatrix}
        x^{ji}\cos\psi^i + y^{ji}\sin\psi^i \\
        -x^{ji}\sin\psi^i + y^{ji}\cos\psi^i \\
        \psi^{ji}
    \end{bmatrix}.
\end{equation}

By differentiating \eqref{eq:x-ego} w.r.t. time, we obtain the first time derivative $\dot{\bm{x}}^{j\underline{i}} \coloneqq [\dot{x}^{j\underline{i}}, \dot{y}^{j\underline{i}}, \dot{\psi}^{j\underline{i}}]^\top$ as 
{\footnotesize
    \begin{align} \label{eq:dot-x-ego}
        \dot{x}^{j\underline{i}} &= \cos\psi^i\, \dot{x}^{ji} - \sin\psi^i\, x^{ji} \dot{\psi}^{i} + \sin\psi^i\, \dot{y}^{ji} + \cos\psi^i\, y^{ji} \dot{\psi}^{i}, \nonumber \\
        \dot{y}^{j\underline{i}} &= \cos\psi^i\, \dot{y}^{ji} - \sin\psi^i\, y^{ji} \dot{\psi}^{i} - \sin\psi^i\, \dot{x}^{ji} - \cos\psi^i\, x^{ji} \dot{\psi}^{i}, \\
        \dot{\psi}^{j\underline{i}} &= \dot{\psi}^{ji}. \nonumber
    \end{align}
}Similarly, applying the product rule to \eqref{eq:dot-x-ego} yields the second time derivative $\ddot{\bm{x}}^{j\underline{i}} \coloneqq [\ddot{x}^{j\underline{i}}, \ddot{y}^{j\underline{i}}, \ddot{\psi}^{j\underline{i}}]^\top$ as 
{\footnotesize
\begin{align} \label{eq:ddot-x-ego}
\ddot{x}^{j\underline{i}} &= \cos\psi^i\, \ddot{x}^{ji} - 2 \sin\psi^i\, \dot{x}^{ji}\dot{\psi}^{i} - x^{ji} \cos\psi^i\, (\dot{\psi}^{i})^2 \nonumber \\
&\quad - x^{ji} \sin\psi^i\, \ddot{\psi}^{i} + \sin\psi^i\, \ddot{y}^{ji} + 2 \cos\psi^i\, \dot{y}^{ji}\dot{\psi}^{i} \nonumber \\
&\quad - y^{ji} \sin\psi^i\, (\dot{\psi}^{i})^2 + y^{ji} \cos\psi^i\, \ddot{\psi}^{i}, \nonumber \\[10pt]
\ddot{y}^{j\underline{i}} &= \cos\psi^i\, \ddot{y}^{ji} - 2 \sin\psi^i\, \dot{y}^{ji}\dot{\psi}^{i} - y^{ji} \cos\psi^i\, (\dot{\psi}^{i})^2 \nonumber \\
&\quad - y^{ji} \sin\psi^i\, \ddot{\psi}^{i} - \sin\psi^i\, \ddot{x}^{ji} - 2 \cos\psi^i\, \dot{x}^{ji}\dot{\psi}^{i} \nonumber \\
&\quad + x^{ji} \sin\psi^i\, (\dot{\psi}^{i})^2 - x^{ji} \cos\psi^i\, \ddot{\psi}^{i}, \nonumber \\[10pt]
\ddot{\psi}^{j\underline{i}} &= \ddot{\psi}^{ji}. \nonumber
\end{align}
}We compute the time derivatives of $\bm{x}^{ji}$ as
\begin{equation}
    \dot{\bm{x}}^{ji} \coloneqq \dot{\bm{x}}^j - \dot{\bm{x}}^i \quad \text{and} \quad \ddot{\bm{x}}^{ji} \coloneqq \ddot{\bm{x}}^j - \ddot{\bm{x}}^i.
\end{equation}
From \eqref{eq:kinematic-bicycle-model}, we obtain $\dot{\bm{x}}^i \coloneqq [\dot{x}^i,\, \dot{y}^i,\, \dot{\psi}^i]^\top \in \mathbb{R}^3$. Differentiating this yields $\ddot{\bm{x}}^i \coloneqq [\ddot{x}^i,\, \ddot{y}^i,\, \ddot{\psi}^i]^\top$ as (omitting the superscript $i$ for brevity)
{\footnotesize
\begin{equation} \label{eq:ddot-x-global}
    \begin{bmatrix}
        \ddot{x}^i \\
        \ddot{y}^i \\
        \ddot{\psi}^i
    \end{bmatrix}
    =
    \begin{bmatrix}
        u_v \cos(\psi + \beta) - v \sin(\psi + \beta) (\dot{\psi} + \dot{\beta}) \\
        u_v \sin(\psi + \beta) + v \cos(\psi + \beta) (\dot{\psi} + \dot{\beta}) \\
        \dfrac{\cos\beta}{\ell_{wb}} \left(u_v \tan\delta + v \sec^2\delta\, u_\delta - v \tan\beta\, \tan\delta\, \dot{\beta}\right)
    \end{bmatrix},
\end{equation}
}where $\beta$ is computed as in \eqref{eq:beta}, $\dot{\beta} = \frac{k \sec^2\delta}{1 + (k \tan\delta)^2} u_\delta$, with $k \coloneqq {\ell_r}/{\ell_{wb}}$ and $\sec\delta \coloneqq {1}/{\cos\delta}$. A similar computation applies to robot $j$.

\begin{theorem} \label{theorem:relative-degree}
Function $h_{\mathrm{MTV}}(\bm{x}^{j\underline{i}})$ defined in \eqref{eq:h} w.r.t. system \eqref{eq:relative-dynamics} has relative degree two.
\end{theorem}
\begin{proof}
We determine the relative degree by counting the number of times $h_{\mathrm{MTV}}(\bm{x}^{j\underline{i}})$ must be differentiated along the system dynamics until the control input $\bm{u}$ appears. The first time derivative is given by
\begin{equation} \label{eq:dot-h}
    \dot{h}_{\mathrm{MTV}}(\bm{x}^{j\underline{i}}) = \dot{h}_\theta(\bm{x}^{j\underline{i}}) = \nabla h_\theta^\top \, \dot{\bm{x}}^{j\underline{i}},
\end{equation}
where $\nabla h_\theta \coloneqq \begin{bmatrix} \frac{\partial h_\theta}{\partial x^{j\underline{i}}},\, \frac{\partial h_\theta}{\partial y^{j\underline{i}}},\, \frac{\partial h_\theta}{\partial \psi^{j\underline{i}}} \end{bmatrix}^\top$ is the gradient vector of $h_\theta$. Since neither $\nabla h_\theta$ nor $\dot{\bm{x}}^{j\underline{i}}$ includes the control input, the first derivative does not depend on $\bm{u}$. Differentiating once more and applying the product rule gives
\begin{equation} \label{eq:ddot-h}
    \ddot{h}_{\mathrm{MTV}}(\bm{x}^{j\underline{i}}) = \nabla h_\theta^\top \, \ddot{\bm{x}}^{j\underline{i}} + \dot{\bm{x}}^{j\underline{i}\top} \, H_{h_\theta}\, \dot{\bm{x}}^{j\underline{i}},
\end{equation}
where $H_{h_\theta} \in \mathbb{R}^{3 \times 3}$ is the Hessian matrix of $h_\theta$. Since $\ddot{\bm{x}}^{j\underline{i}}$ is computed from $\ddot{\bm{x}}^i$ and $\ddot{\bm{x}}^j$, which explicitly involve the control inputs $u_v$ and $u_\delta$ (see \eqref{eq:ddot-x-global}), the control input appears in the second derivative. Thus, $h_{\mathrm{MTV}}(\bm{x}^{j\underline{i}})$ has relative degree two.
\end{proof}

Note that for a given neural network and input vector, the gradient and Hessian can be computed directly. We omit the details due to space limitations.

\begin{remark}
Although we use the kinematic bicycle model in this derivation, the proposed notion of relative dynamics can be applied to any dynamical model. The derivation procedure is similar to our derivation in this section.
\end{remark}

\subsection{Optimal Control Problem Formulation}  \label{sec:ocp}
\subsubsection{Training the Safety-Margin Approximator}
We generate a training dataset by computing the MTV-based safety margin using \Cref{alg:mtv}. We create a dense grid of inputs over the feature space as
\[
[x^{j\underline{i}},\, y^{j\underline{i}},\, \psi^{j\underline{i}}] \in [-3\ell_{wb},\, 3\ell_{wb}] \times [-3\ell_{wb},\, 3\ell_{wb}] \times [-\pi,\, \pi].
\]
We limit the relative position of robot $j$ to this space to allow for an estimable upper bound on the approximation error (see \cref{ass:error-upper-bound}). We fix robot $i$'s position and heading at zero and densely enumerate robot $j$'s position and heading over the feature space. Then, we run \cref{alg:mtv} with the positions and headings of the robots to obtain the MTV-based distances, which serve as labels. Our training dataset contains approximately 80k uniformly distributed data points. We train a simple fully connected neural network $h_\theta$ with two hidden layers of 62 nodes each and \texttt{TanH} activation functions. We test on a separate dataset of 20k points, yielding a maximum approximation error of \SI{0.0121}{\meter} (corresponding to \SI{15.2}{\percent} of the robot's width) and a mean error of \SI{2.78}{\percent} of the robot's width. Increasing the size of the testing dataset produces similar results.

\subsubsection{\ac{ocp}}
We use the learned safety-margin approximator $h_\theta$ as the candidate \ac{cbf} and incorporate it into a \ac{qp} to modify an unsafe nominal controller, resulting in a \ac{cbf}-\ac{qp} formulation. The goal is to adjust the nominal control actions minimally to ensure safety. We employ an \ac{rl} policy trained using our \texttt{SigmaRL} \cite{xu2024sigmarl, xu2024xpmarl}, an open-source multi-agent \ac{rl} framework for motion planning of \acp{cav}, as the nominal controller. Henceforth, we refer to the robots as \textit{vehicles}. At each time step, the nominal controller receives a short-term waypoint-based reference path and outputs control actions to follow it. We purposely train the policy to be greedy in following the reference path without considering collisions. We formulate the \ac{ocp} as a \ac{cbf}-\ac{qp} as

\begin{subequations} \label{eq:ocp}
\begin{align}
    J(\bm{u}) &= \min_{\bm{u}} (\bm{u} - \bm{u}_{\text{nom}})^\top Q (\bm{u} - \bm{u}_{\text{nom}}), \\
    \text{s.t.} \quad & \Psi_2(\bm{x}^{j\underline{i}}) \geq 0 \quad (\text{see \eqref{eq:hocbf-constraint}}), \label{constraint:2nd-cbf-condition} \\
    & \bm{u}_{\text{min}} \leq \bm{u} \leq \bm{u}_{\text{max}}, \label{constraint:u-limits}
\end{align}
\end{subequations}
where $\bm{u}_{\text{nom}}$ denotes the nominal control action, and $Q \in \mathbb{R}^{4 \times 4}$ is a weighting matrix. Constraint \eqref{constraint:2nd-cbf-condition} is the \ac{cbf} condition for collision avoidance. We use $\Psi_2$ since the candidate \ac{cbf} has relative degree two, as shown in \Cref{theorem:relative-degree}. For simplicity, we choose the same linear class $\mathcal{K}$ functions, with $\alpha_1(h) = k_\alpha h$ and $\alpha_2(h) = k_\alpha h$, where $k_\alpha > 0$, yielding
\[
\Psi_2(\bm{x}^{j\underline{i}}) \coloneqq \ddot{h}_{\mathrm{MTV}}(\bm{x}^{j\underline{i}}) + 2 k_\alpha \dot{h}_{\mathrm{MTV}}(\bm{x}^{j\underline{i}}) + k_\alpha^2 h_{\mathrm{MTV}}(\bm{x}^{j\underline{i}}).
\]
Constraint \eqref{constraint:u-limits} enforces the control inputs within their physical limits, with $\bm{u}_{\text{min}}$ and $\bm{u}_{\text{max}}$ representing the lower and upper bounds. We solve \eqref{eq:ocp} iteratively in a discrete-time manner using the Python package \texttt{CVXPY} \cite{diamond2016cvxpy}.

\Cref{tab:parameters} lists the simulation parameters, and \Cref{fig_diagram} shows the flow diagram of our \ac{cbf}-\ac{qp} formulation using the learned MTV-based safety margin.

We also compare our MTV-based safety margin with the traditional \ac{c2c}-based safety margin. For the latter, we use
\begin{equation} \label{eq:c2c}
    h_{\mathrm{\ac{c2c}}}(\bm{x}^{j\underline{i}}) \coloneqq \sqrt{\left(x^{j\underline{i}}\right)^2 + \left(y^{j\underline{i}}\right)^2} - 2r_{\text{min}},
\end{equation}
where $r_{\text{min}} \coloneqq {\sqrt{\ell^2 + w^2}}/{2}$ is the minimum radius required to enclose the vehicle (with $\ell$ and $w$ being the vehicle's length and width, respectively). One can easily verify that \eqref{eq:c2c} is a valid \ac{hocbf} with relative degree two.

\begin{table}[t]
    \caption{Simulation parameters.}
    \centering
    \begin{tabular}{ll}
        \toprule
        Parameter & Value \\
        \midrule
        Length $\ell$, width $w$ & \SI{0.16}{\meter}, \SI{0.08}{\meter} \\
        Wheelbase $\ell_{wb}$, rear wheelbase $\ell_r$ & \SI{0.16}{\meter}, \SI{0.08}{\meter} \\
        Max. (min.) acceleration $u_{v,\text{max}}$ & \SI{20}{\meter\per\second\squared} (\SI{-20}{\meter\per\second\squared}) \\
        Max. (min.) steering rate $u_{\delta,\text{max}}$ & \SI{16}{\radian\per\second} (\SI{-16}{\radian\per\second}) \\
        Weighting matrix $Q$ & $\mathcal{I}_{4 \times 4}$ (identity matrix) \\
        \bottomrule
    \end{tabular}
    \label{tab:parameters}
\end{table}

\begin{figure}[t]
    \centering
    \includegraphics[width=0.8\linewidth]{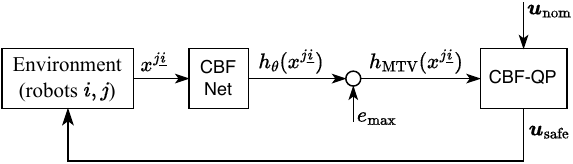}
    \caption{Flow diagram of the \ac{cbf}-\ac{qp} formulation with the MTV-based safety margin.}
    \label{fig_diagram}
\end{figure}

\subsection{First Experiment: Overtaking} \label{sec:scenario-1}
In this scenario, vehicle $i$ (blue) is the ego vehicle attempting to overtake a slower-moving vehicle $j$ (green) ahead, as depicted in \cref{fig:eva_cbf_overtaking_c2c} and \cref{fig:eva_cbf_overtaking_mtv}. Vehicle $i$ uses a trained \ac{rl}-based nominal controller that directs it to move at approximately \SI{1.0}{\meter\per\second} and a \ac{cbf} to ensure safety. Conversely, vehicle $j$ only uses a trained \ac{rl}-based nominal controller, which maintains a constant speed of \SI{0.5}{\meter\per\second}. To encourage overtaking, vehicle $i$’s reference path is projected to the centerline of the adjacent lane. In addition, to test the robustness of the \ac{cbf}, vehicle $j$ conditionally switches lanes to obstruct the overtaking maneuver until it has obstructed three times, after which it remains in its lane. In this scenario, because vehicle $j$ follows its nominal actions, the \ac{ocp} \eqref{eq:ocp} has only two decision variables corresponding to vehicle $i$'s control actions. At each time step, vehicle $j$ communicates its nominal actions to vehicle $i$, which uses them to formulate Constraint \eqref{constraint:2nd-cbf-condition} for collision avoidance. We ignore communication delays.

\textbf{Results}: \Cref{fig:eva_cbf_overtaking_c2c} shows the performance with the \ac{c2c}-based safety margin. Vehicle $i$ starts at $x=\SI{-1.2}{\meter}$ and vehicle $j$ at $x=\SI{-0.4}{\meter}$. During the three obstructive maneuvers, the \ac{c2c}-based margin prevents a collision, but after $t=\SI{4.8}{\second}$, vehicle $j$ stops obstructing, and vehicle $i$ is unable to complete the overtaking maneuver because of the conservatism of the \ac{c2c}-based margin. The blue line in \cref{fig:eva_cbf_overtaking_c2c} shows the $h$ value over time, with near-zero values indicating that the system state remains close to the boundary of the safe set. In contrast, as shown in \cref{fig:eva_cbf_overtaking_mtv}, using our MTV-based safety margin, vehicle $i$ successfully overtakes vehicle $j$ at $t=\SI{6.4}{\second}$.

\begin{figure}[t]
	\centering
	\includegraphics[width=0.48\textwidth]{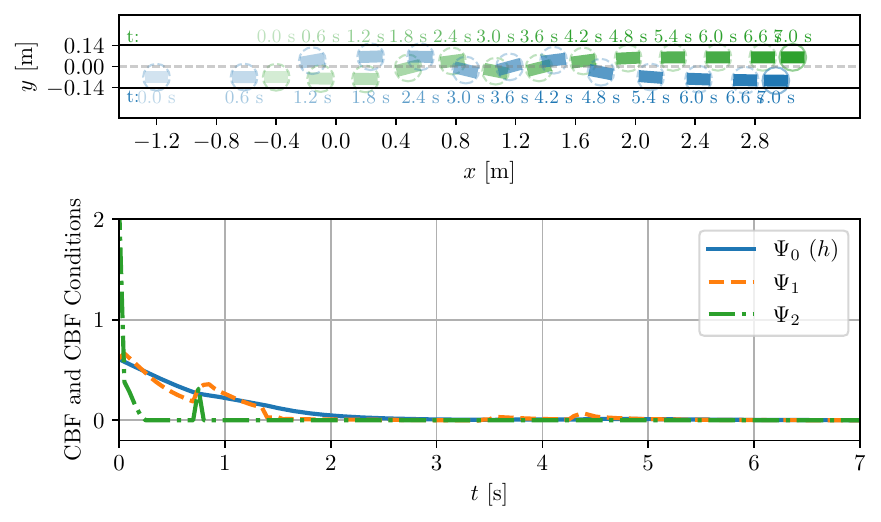}
	\caption{Overtaking scenario with \ac{c2c}-based safety margin. The lowest safety margin occurs for $t \ge \SI{5.4}{\second}$.}
	\label{fig:eva_cbf_overtaking_c2c}
\end{figure}
\begin{figure}[t]
	\centering
	\includegraphics[width=0.48\textwidth]{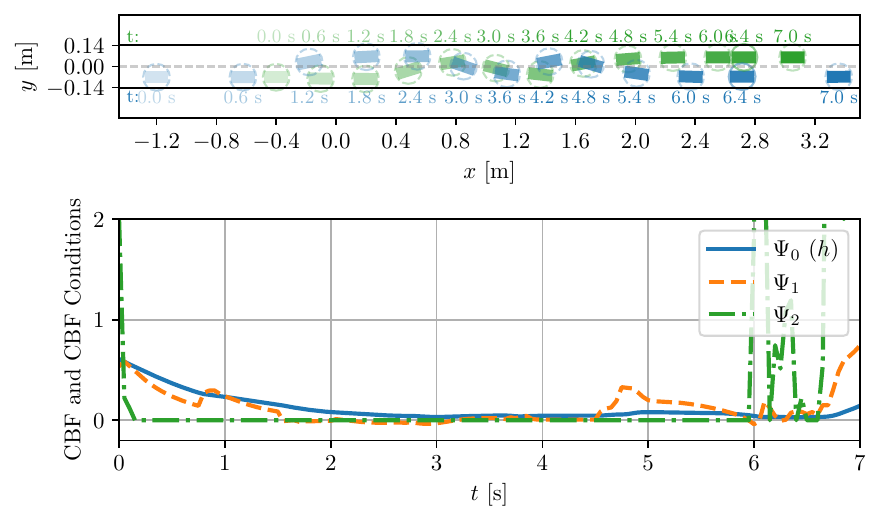}
	\caption{Overtaking scenario with MTV-based safety margin. The lowest safety margin occurs at $t=\SI{6.4}{\second}$.}
	\label{fig:eva_cbf_overtaking_mtv}
\end{figure}

\subsection{Second Experiment: Bypassing} \label{sec:scenario-2}
In this scenario, two vehicles approach each other from opposite directions on a narrow road, as depicted in \cref{fig:eva_cbf_bypassing_c2c} and \cref{fig:eva_cbf_bypassing_mtv}. Both vehicles use the same \ac{rl}-based nominal controller that controls them to move at a speed of \SI{1.0}{\meter\per\second}. We apply the \ac{cbf} to both vehicles to adjust the nominal control actions to ensure collision-freeness. Therefore, the centralized \ac{ocp} \eqref{eq:ocp} has four decision variables, two for each vehicle's actions.

Let $Y = y_0$ denote a horizontal line at $y_0 \in \mathbb{R}$. Initially, we project the nominal reference points for both vehicles onto $Y^i = 0$ and $Y^j = 0$, respectively. As the vehicles approach, the projections are shifted to $Y^i = y_{\text{nom}} > 0$ and $Y^j = -y_{\text{nom}} < 0$ to encourage bypassing. The parameters $y_{\text{nom}}$ and $k_\alpha$ are tuned jointly for optimal bypassing with minimal lateral displacement. Final values are $y_{\text{nom}} = \SI{0.116}{\meter}$ (\SI{145.3}{\percent} of the vehicle width) and $k_\alpha = 3$ for the \ac{c2c}-based margin, and $y_{\text{nom}} = \SI{0.072}{\meter}$ (\SI{90.0}{\percent} of the vehicle width) and $k_\alpha = 6$ for the MTV-based margin. 

\textbf{Results}: As shown in \cref{fig:eva_cbf_bypassing_c2c}, vehicle $i$ starts at $x=\SI{-1.2}{\meter}$ moving rightward and vehicle $j$ at $x=\SI{1.2}{\meter}$ moving leftward. At $t=\SI{2.5}{\second}$, the safety margin reaches its minimum and the vehicles bypass each other, albeit with significant lateral evasion (vehicle $i$: \SI{119.6}{\percent} of the vehicle width, vehicle $j$: \SI{124.7}{\percent}; average: \SI{122.2}{\percent}). The bypassing process, i.e., both vehicles reach their opposite sides, takes about \SI{3.6}{\second}. In comparison, with our MTV-based safety margin, the minimum safety margin occurs at $t=\SI{1.7}{\second}$, and the vehicles bypass successfully with an average lateral evasion of \SI{83.13}{\percent} of the vehicle width (a reduction of \SI{33.5}{\percent} compared to the \ac{c2c}-based safety margin), and the bypass completes at $t=\SI{3.0}{\second}$, which is \SI{16.7}{\percent} faster.

\begin{figure}[t]
	\centering
	\includegraphics[width=0.48\textwidth]{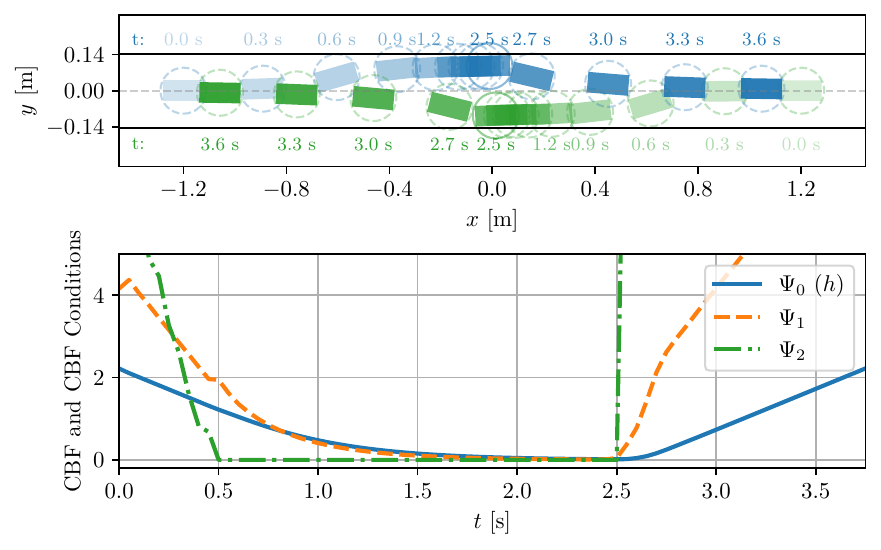}
	\caption{Bypassing scenario with \ac{c2c}-based safety margin. The minimum safety margin occurs at $t=\SI{2.5}{\second}$.}
	\label{fig:eva_cbf_bypassing_c2c}
\end{figure}
\begin{figure}[t]
	\centering
	\includegraphics[width=0.48\textwidth]{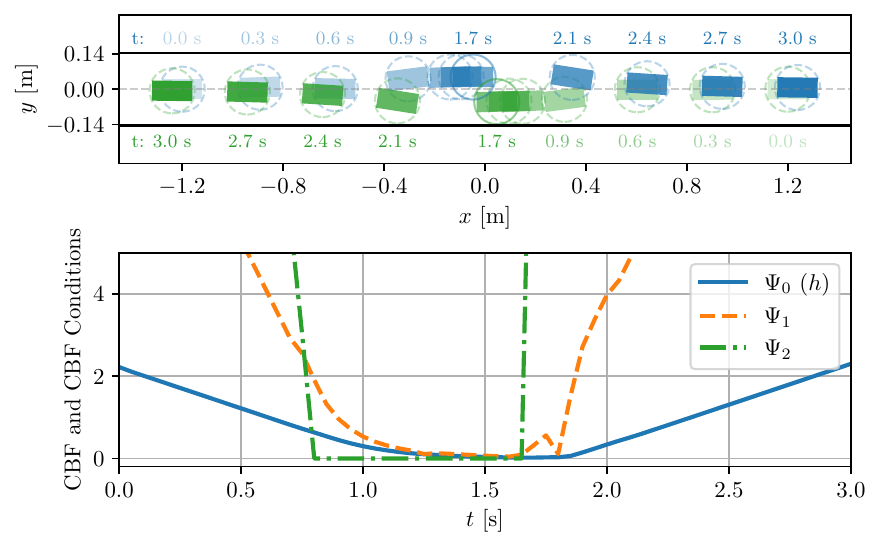}
	\caption{Bypassing scenario with MTV-based safety margin. The minimum safety margin occurs at $t=\SI{1.7}{\second}$.}
	\label{fig:eva_cbf_bypassing_mtv}
\end{figure}

\subsection{Discussions and Limitations} \label{sec:limitation}
The overtaking and bypassing scenarios demonstrate that our MTV-based safety margin yields less conservative collision avoidance compared to the traditional \ac{c2c}-based approach. In the overtaking scenario, the \ac{c2c}-based safety margin prevents collisions but restricts overtaking, while our MTV-based safety margin allows a smooth, safe, and efficient overtaking maneuver. In the bypassing scenario, both methods maintain safety, yet our MTV-based margin reduces the lateral space required for bypassing by \SI{33.5}{\percent} and decreases bypassing time by \SI{16.7}{\percent}. Importantly, these improvements do not significantly affect the computation time. In the overtaking scenario, the \ac{ocp} \eqref{eq:ocp} is solved in an average of \SI{7.4}{\milli\second} per step with the \ac{c2c}-based margin and \SI{7.6}{\milli\second} with the MTV-based margin, while in the bypassing scenario, the average times are \SI{11.3}{\milli\second} and \SI{11.6}{\milli\second}, respectively.

Applying our \ac{hocbf} requires computing the gradient and Hessian matrix of the safety margin approximator (see \eqref{eq:dot-h} and \eqref{eq:ddot-h}). Since the original function for the MTV-based margin is non-differentiable, we can not compute their actual values and do not know how accurate they are. Nevertheless, given the marginal approximation error in the safety margin (\SI{1.4}{\percent} of the robot's width on average) and its continuous nature, we expect the gradient and Hessian approximation errors to be similarly minor and negligible.

\section{Conclusions}\label{sec:conclusions}
We proposed \iac{mtv}-based safety margin for collision avoidance of car-like robots that accounts for their actual geometries and headings, offering a more precise estimation of safe regions compared to the traditional \ac{c2c}-based safety margin that approximates robots as circles. Because this safety margin is inherently non-differentiable, we approximated it with a neural network and introduced a notion of relative dynamics to facilitate tractable learning. We provided the theoretical foundation for applying this notion to a nonlinear kinematic bicycle model and compared our MTV-based safety margin with the traditional \ac{c2c}-based safety margin through numerical experiments in overtaking and bypassing scenarios. In the overtaking scenario, while the traditional approach failed to complete the maneuver, our method succeeded. In the bypassing scenario, our approach reduced the required lateral space for bypassing by \SI{33.5}{\percent} and shortened the bypass time by \SI{16.7}{\percent}. Our approach achieved all these enhancements without significantly increasing computation time. Future work will extend our approach to multi-robot systems.
    
\bibliographystyle{IEEEtran}
\bibliography{references_bassam, references_others}

\end{document}